\documentclass{styles/bmvc2k}
\usepackage{booktabs}
\usepackage{tabularx}
\usepackage{multirow}
\usepackage{algorithm}
\usepackage{float}
\usepackage{algpseudocode}
\usepackage{amsmath,amsfonts,amsthm} 
\usepackage{tikz}
\usepackage{wrapfig}
\usepackage[dvipsnames]{xcolor}
\usetikzlibrary{positioning,arrows.meta,calc,fit,backgrounds}
\theoremstyle{definition}
\newtheorem{common}{Common}[section]

\newtheorem{lemma}[common]{Lemma}

\newtheorem{proposition}[common]{Proposition}

\title{Improving Certified Robustness via Adversarial Distillation}

\addauthor{Matteo Melis}{mmelis01@qub.ac.uk}{1}
\addauthor{Jesús Martínez Del Rincón}{j.martinez-del-rincon@qub.ac.uk}{1}
\addauthor{Vishal Sharma}{V.Sharma@qub.ac.uk}{1}
\addinstitution{
 Centre for Secure Information Technologies (CSIT),\\
 Queen's University Belfast
}

\runninghead{Melis \etal}{Certified Robustness via Adversarial Distillation}

\def\eg{\emph{e.g}\bmvaOneDot, }

\def\ie{\emph{i.e}\bmvaOneDot, }
\def\etal{\emph{et al}\bmvaOneDot}
\DeclareMathOperator*{\argmax}{argmax}

\newcommand{\pos}[1]{\textcolor{ForestGreen}{+#1}}
\newcommand{\negv}[1]{\textcolor{BrickRed}{#1}}
\newcommand{\neutral}[1]{\textcolor{black}{#1}}

\begin{document}

\maketitle

\begin{abstract}
Certified training aims to produce models whose predictions can be formally verified against adversarial perturbations, typically by optimising upper bounds on the worst-case loss over an allowed perturbation set. For neural networks, certified training methods based purely on tight relaxation bounds produce networks that are amenable to certification, but sacrifice standard accuracy. Conversely, adversarial training often yields stronger empirical robustness and standard accuracy, but the resulting models are generally difficult to certify with neural network verifiers. Recently, the literature has shown that better standard-certified accuracy trade-offs can be achieved by combining adversarial training objectives with loose over-approximations based on Interval Bound Propagation (IBP), effectively interpolating between lower and upper bounds of the worst-case loss. Building on this, we introduce AD-CERT, a certified training objective that combines adversarial distillation with an IBP upper bound. We show that distilling adversarial information over the logit space from an empirically robust teacher provides an effective lower bound surrogate for certified training, with AD-CERT achieving state-of-the-art certified performance on several robustness benchmarks for ReLU networks. Furthermore, in a unified setup, distilling adversarial information at the logit-level is shown to improve certified accuracy over a robust feature-space distillation objective by up to 5.40 percentage points. Our code has been made available at \url{https://github.com/matteo-melis/ad-cert}.
\end{abstract}
\section{Introduction}
While deep neural networks have seen great success across various disciplines, adversarial examples \cite{biggio2023, szegedy2014, goodfellow2015} raise important questions about their \textit{adversarial robustness}, \ie a network's ability to preserve its prediction under small input perturbations. This is of particular concern in safety-critical domains such as autonomous driving and medical diagnostics. Motivated by this, neural network verification methods \cite{katz2017, ehlers2017} aim to provide formal certificates of robustness for a given network. \\ \\
Neural network verifiers broadly fall into two categories, complete verifiers \cite{katz2017, tjeng2019}, which compute exact bounds but have an exponential worst-case runtime, and incomplete verifiers \cite{zhang2018, singh2019}, which rely on convex relaxations to obtain approximate bounds. Modern state-of-the-art verifiers \cite{ferrari2022complete, xu2021fast, wang2021beta, zhang2022general} typically combine the two, using convex relaxations to accelerate complete verification within a branch-and-bound framework \cite{bunel2020branchboundpiecewiselinear}. \\ \\
Since verification only certifies robustness after training, robust training methods are needed to produce networks that are both accurate and certifiable. Training neural networks for adversarial robustness is typically approached either empirically, through adversarial examples, or through certified training, using over-approximations of the network's reachable set over a perturbation region. Adversarial training \cite{madry2018deeplearningmodelsresistant} is an empirical approach that aims to improve a network's adversarial robustness by augmenting the training objective with adversarial examples, thereby optimising a lower bound on the worst-case loss. However, adversarially trained networks are hard to formally verify and often fall short when faced with stronger or adaptive attack strategies \cite{croce2020, tramer2020}. In contrast, certified training directly optimises an upper bound on the worst-case loss, producing networks that are more amenable to verification, often at the cost of standard accuracy. While earlier certified training methods \cite{mirman2018, gowal2018, zhang2020, shi2021fastcertifiedrobusttraining} trained networks against sound upper bounds of the robust loss, recent state-of-the-art methods \cite{mao2023, muller2023certifiedtrainingsmallboxes, palma2024, depalma2026} effectively optimise unsound approximations of the worst-case loss by combining adversarial training with loose IBP-based over-approximations \cite{gowal2018}. Due to IBP's favourable optimisation properties \cite{jovanovic2022, mao2024} and adversarial training's empirical strengths, such certified training methods have been shown to provide state-of-the-art standard-certified accuracy trade-offs in recent benchmarks \cite{mao2025}. \\ \\
Motivated by the success of robust objectives that combine adversarial and certified training, we hypothesise that transferring adversarial knowledge from an empirically robust teacher to a certified student using knowledge distillation could further improve certified performance when paired with IBP bounds. This leads to the following contributions:\par\noindent
\begin{figure*}
\centering
\fbox{\begin{tikzpicture}[
    font=\footnotesize,
    >=Latex,
    node distance=7mm and 8mm,
    every node/.style={align=center},
    box/.style={
        draw,
        rounded corners,
        minimum height=7mm,
        inner sep=3pt
    },
    io/.style={
        box,
        minimum width=11mm,
        fill=white
    },
    teacher/.style={
        box,
        minimum width=20mm,
        fill=blue!6
    },
    student/.style={
        box,
        minimum width=18mm,
        fill=green!10
    },
    cert/.style={
        box,
        minimum width=22mm,
        fill=orange!12
    },
    aux/.style={
        box,
        minimum width=18mm,
        fill=red!10
    },
    loss/.style={
        box,
        minimum width=30mm,
        fill=gray!15
    },
    sum/.style={
        draw,
        circle,
        inner sep=0pt,
        minimum size=5mm
    },
    group/.style={
        draw,
        dashed,
        rounded corners,
        inner sep=4pt
    },
    klloss/.style={
        box,
        minimum width=30mm,
        fill=blue!6
    },
    ibploss/.style={
        box,
        minimum width=30mm,
        fill=orange!10
    },
    img/.style={inner sep=0pt, outer sep=0pt}]
    
    \node[klloss] (klloss)
    {\((1-\alpha)\,\mathcal{L}_{\mathrm{KL}}\!\left(P_{\theta_t}(\mathbf{x})\,\|\,P_{\theta_s}(\mathbf{x_{adv}})\right)\)};
    
    \node[sum, right=10mm of klloss] (plus) {\(+\)};
    
    \node[ibploss, right=10mm of plus] (ibploss)
    {\(\alpha\,\mathcal{L}_{\mathrm{CE}}\!\left(-\underline{\mathbf{z}}^{\Delta}_{\theta_s}(\mathbf{x},y),\,y\right)\)};
    
    \node[loss, above=4mm of plus, minimum width=26mm, fill=gray!15] (total)
    {\(\mathcal{L}_{\mathrm{AD\mbox{-}CERT}}\)};
    
    \draw[->] (klloss) -- (plus);
    \draw[->] (ibploss) -- (plus);
    \draw[->] (plus) -- (total);
    
    \node[io, below left=5mm and 2mm of klloss] (ps)
    {\(P_{\theta_s}(\mathbf{x_{adv}})\)};
    
    \node[io, below=5mm of klloss] (pt)
    {\(P_{\theta_t}(\mathbf{x})\)};
    
    \node[student, below=4mm of ps] (student)
    {Student\\\(f_{\theta_s}\)};
    
    \node[teacher, below=4mm of pt] (teacher)
    {Frozen teacher \\\(f_{\theta_t}\)};
    
    \node[aux, below=4mm of student] (pgd)
    {PGD attack};
    
    \node[img, below=14mm of teacher] (xleftimg) {
        \includegraphics[width=10mm]{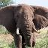}
    };
    
    \node[below=1mm of xleftimg] {\normalsize $\mathbf{x}$};
    
    \node[left=7mm of ps, rotate=90, text=blue!70!black] (emplabel)
    {\scriptsize Adversarial distillation};
    
    \draw[->] (xleftimg.north) |- (pgd.east);
    \draw[->] (xleftimg) -- (teacher);
    \draw[->] (pgd) -- (student);
    \draw[->] (student) -- (ps);
    \draw[->] (teacher) -- (pt);
    \draw[->] (ps) |- (klloss.west);
    \draw[->] (pt) -- (klloss);
    
    \begin{scope}[on background layer]
        \node[group, fit=(emplabel)(ps)(pt)(student)(teacher)(pgd)] {};
    \end{scope}
    
    \node[io, below=5mm of ibploss] (zbound)
    {\(\underline{\mathbf{z}}^{\Delta}_{\theta_s}(\mathbf{x},y)\)};
    
    \node[cert, below=4mm of zbound] (ibpnode)
    {IBP};
    
    \node[student, below=4mm of ibpnode] (ibpstudent)
    {Student\\\(f_{\theta_s}\)};
    
    \node[io, right=40mm of xleftimg] (xbox)
    {\(\mathcal{B}_{\epsilon}(\mathbf{x})=[\mathbf{x}-\epsilon,\mathbf{x}+\epsilon]\)};
    
    \node[left=9mm of zbound, rotate=90, text=orange!80!black] (certlabel)
    {\scriptsize Bound propagation};
    
    \draw[->] (zbound) -- (ibploss);
    \draw[->] (ibpnode) -- (zbound);
    \draw[->] (ibpstudent) -- (ibpnode);
    \draw[->] (xleftimg) -- (xbox);
    \draw[->] (xbox) -- (ibpstudent);
    
    \begin{scope}[on background layer]
        \node[group, fit=(certlabel)(ibpnode)(zbound)(ibpstudent)] {};
    \end{scope}
    \label{tikz:ad-cert-overview}
    \end{tikzpicture}
}
\caption{Overview of AD-CERT (§\ref{sub-sec:ad-cert}) loss formulation, which combines adversarial distillation (blue) with IBP (orange). Note that the detached teacher is previously adversarially trained and the PGD attack is computed against the student's parameters.}
\label{fig:ad-cert-overview}
\end{figure*}
\begin{itemize}
\vspace{-0.4cm}
    \item We introduce \textit{Adversarial Distillation for Certification} (AD-CERT) (§\ref{sub-sec:ad-cert}), a novel certified training objective that combines adversarial logit-level distillation from an empirically robust teacher with IBP bounds (depicted in Figure \ref{fig:ad-cert-overview}).
    \item We show that AD-CERT is a scalarisation between a teacher-guided lower bound surrogate and a certified IBP upper bound on the robust loss (§\ref{sub-sec:lower-bound-surrogate}).
    \item We present extensive experimental evaluations of AD-CERT across standard certified training benchmarks (§\ref{sub-sec:main-results}), showing that it achieves state-of-the-art certified accuracy against current certified training methods. 
    \item We provide a systematic analysis of multiple knowledge distillation approaches for empirical endpoints (§\ref{sub-sec:teacher-student-discussion}), showing that adversarial logit-level distillation preserves teacher robustness more effectively than clean or feature-space distillation over the same architecture, and translates best into certified training when combined with an IBP loss (§\ref{sub-sec:main-results}).
\end{itemize}
\section{Background}
\label{sec:background}
Here, we discuss some necessary background for this work. Let $f_\theta$ denote a multi-class neural network classifier parametrised by $\theta$ such that $f_\theta: \mathbb{R}^{d} \mapsto{\mathbb{R}^C}$ maps an input $\mathbf{x} \in \mathcal{X} \subseteq \mathbb{R}^d$ to a numerical prediction score $[f_\theta(\mathbf{x})]_i$ for each class $i \in \{1, 2, \dots, C\}$. Let $(\mathbf{x}, y) \sim \mathcal{D}$ denote a sample input and true label from a data distribution $\mathcal{D}$ and let $\mathcal{B}_{\varepsilon_p}(\mathbf{x})$ denote a perturbation set based on an $\ell_p$ norm threat model. In this work, as standard with certified training \cite{gowal2018, balunovic2020colt, palma2022, mao2023, muller2023certifiedtrainingsmallboxes, palma2024, depalma2026}, we only consider the
$\ell_\infty$ case, and therefore put $\mathcal{B}_\varepsilon(\mathbf{x})
:= \{\mathbf{x}' : \|\mathbf{x}' - \mathbf{x}\|_\infty \le \varepsilon\}$, where $\mathbf{x}'$ denotes the perturbed input and $\varepsilon$ is the perturbation radius.
\subsection{Neural Network Verification}
A network is said to be \emph{adversarially robust} for an input $\mathbf{x}$ if for all $\mathbf{x'} \in \mathcal{B}_\varepsilon(\mathbf{x})$, we have $[f_\theta(\mathbf{x'})]_i < [f_\theta(\mathbf{x'})]_y \ \forall \ i \ne y$. Let $\mathbf{z}^\Delta_\theta \in \mathbb{R}^{C-1}$ denote the logit-difference between the true class and all other classes, \ie $\mathbf{z}^\Delta_\theta(\mathbf{x}, y) := ([f_\theta(\mathbf{x})]_y -  [f_\theta(\mathbf{x})]_i)_{i\ne y}$. Then, a network's adversarial robustness can be verified by solving:
\begin{equation}
    \min_{\mathbf{x'} \in \mathcal{B}_\varepsilon(\mathbf{x})}\text{ }\min_{i \ne y}{[\mathbf{z}^\Delta_\theta(\mathbf{x'}, y)]}_i,
    \label{eq:rob_ver}
\end{equation}
and checking if the solution to Equation \eqref{eq:rob_ver} is greater than zero. Finding an exact solution to Equation \eqref{eq:rob_ver} was shown to be NP-hard \cite{katz2017}, thus, neural network verifiers either approximate the bounds \cite{zhang2018, singh2019}, or perform complete verification at a larger computational cost \cite{katz2017, tjeng2019, ferrari2022complete, xu2021fast, wang2021beta, zhang2022general}.

\subsection{Training for Robustness}
\label{sub-sec:train-rob}
Standard training of neural network classifiers optimises network parameters $\theta$ by minimising the expected cross-entropy loss, \ie
\begin{equation}
    \min_\theta \mathbb{E}_{(\mathbf{x}, y) \sim \mathcal{D}} \left[ \mathcal{L}_{\mathrm{CE}} \left(f_\theta ( \mathbf{x}), y \right) \right],
    \label{eq:train_std}
\end{equation}
where $\mathcal{L}_{\mathrm{CE}}(f_\theta(\mathbf{x}), y) = \log (1 + \sum_{i\ne y}^C\exp ([f_\theta(\mathbf{x})]_i -  [f_\theta(\mathbf{x})]_y))$. In contrast, 
training for robustness can be viewed as the following min-max objective:
\begin{equation}
    \min_\theta \mathbb{E}_{(\mathbf{x}, y) \sim \mathcal{D}} \left[ \max_{\mathbf{x}' \in \mathcal{B}_\varepsilon( \mathbf{x})} \mathcal{L}_{\mathrm{CE}} \left(f_\theta ( \mathbf{x}'), y \right) \right].
    \label{eq:train_rob}
\end{equation}
Concretely, for an input $\mathbf{x}$, the inner maximisation finds the most malicious perturbation $\mathbf{x}' \in \mathcal{B}_\varepsilon(\mathbf{x})$ that maximises the loss. The outer minimisation seeks the optimal network parameters $\theta$ such that these threats are less effective against the network predictions. Since the inner maximisation problem is non-convex, this is generally solved by under- or over-approximating the worst-case loss. In general, if we let $\mathcal{L}^*_\mathrm{rob}(f_\theta(\mathbf{x}),y) := \displaystyle \max_{\mathbf{x}' \in \mathcal{B}_\varepsilon( \mathbf{x})} \mathcal{L} \left(f_\theta ( \mathbf{x}'), y \right)$ denote a network objective solved for the optimisation problem in Equation \eqref{eq:train_rob}, then we have:
\begin{equation}
    \underline{\mathcal{L}}_\mathrm{rob}(f_\theta(\mathbf{x}),y) \le \mathcal{L}^*_\mathrm{rob}(f_\theta(\mathbf{x}),y) \le \overline{\mathcal{L}}_\mathrm{rob}(f_\theta(\mathbf{x}),y),
    \label{eq:rob-inequality}
\end{equation}
with $\underline{\mathcal{L}}_\mathrm{rob}(f_\theta(\mathbf{x}),y)$ and $\overline{\mathcal{L}}_\mathrm{rob}(f_\theta(\mathbf{x}),y)$ corresponding to adversarial and certified training, respectively, as described below.
\paragraph{Adversarial Training (AT)} AT first approximates a solution to the inner maximisation problem of Equation \eqref{eq:train_rob} by performing an adversarial attack, such as projected gradient descent (PGD) \citep{madry2018deeplearningmodelsresistant}, within $\mathcal{B}_\varepsilon(\mathbf{x})$ for each input $\mathbf{x}$. Following this, the outer minimisation problem of Equation \eqref{eq:train_rob} is solved by augmenting the network's training objective with the adversarial example $\mathbf{x_{adv}} \in \mathcal{B}_\varepsilon(\mathbf{x})$ obtained by the attack. In its standard form, AT aims to minimise:
\begin{equation}
    \mathcal{L}_{\mathrm{AT}} := \mathcal{L}_{\mathrm{CE}} \left(f_\theta ( \mathbf{x_{adv}}), y \right).
\end{equation}
This can be seen as optimising a lower bound of the worst-case loss, \ie $\underline{\mathcal{L}}_\mathrm{rob}(f_\theta(\mathbf{x}), y)$ from Equation \eqref{eq:rob-inequality}. While AT leads to strong empirical robustness, it produces complex networks that are difficult to formally verify and remain susceptible to stronger attacks \cite{tramer2020,croce2020}.
\paragraph{Certified Training and Bound Propagation} Certified training optimises an upper bound of the inner-maximisation problem in Equation \eqref{eq:train_rob} by over-approximating the worst-case loss.
A simple approach is Interval Bound Propagation (IBP) \cite{gowal2018}, leveraging interval arithmetic to approximate the output range of a given layer. Given an input $\mathbf{x}$, IBP propagates the upper and lower bounds of the perturbed input space, $\mathcal{B}_\varepsilon(\mathbf{x})$, (corresponding to $\overline{\mathbf{x}} := \mathbf{x} + \varepsilon$ and $\underline{\mathbf{x}} := \mathbf{x} - \varepsilon$, respectively) through the entire network layer by layer, providing a hyper-rectangle that is monotonically increasing every layer. The final layer output of this propagation, which we denote with \textsc{Box}, is an over-approximation of the possible output space, \ie \textsc{Box} $ \supseteq f_\theta(\mathcal{B}_\varepsilon( \mathbf{x}))$. Hence, using the \textsc{Box} relaxation, we can compute the worst-case logit difference between the true class and all other classes by putting $\mathbf{\underline{z}}_\theta^\Delta(\mathbf{x}, y) := \left(\underline{[f_{\theta}(\mathbf{x}')]}_y - \overline{[f_{\theta}(\mathbf{x}')]}_i \right)_{i \ne y}$ and obtain a robust loss given by:
\begin{equation}
    \mathcal{L}_{\mathrm{IBP}} := \mathcal{L}_{\mathrm{CE}}(-\mathbf{\underline{z}}_{\theta}^\Delta(\mathbf{x},y),y),
\end{equation}
where $\mathcal{L}_{\mathrm{IBP}}$ is a sound upper bound, $\overline{\mathcal{L}}_\mathrm{rob}(f_\theta(\mathbf{x}), y)$, from Equation \eqref{eq:rob-inequality}. Further details on IBP can be found in Appendix \ref{app:ibp}.
\\ \\
Interestingly, alternative certified training methods that utilise tighter relaxations, \eg linear relaxations \cite{wong2018polytope, zhang2020}, consistently yield worse performance than imprecise IBP-based training. \citet{jovanovic2022} attribute this phenomenon to tighter relaxations adding discontinuity and difficulty to the optimisation objective. Moreover, \citet{mao2024} explicitly show that IBP-based training leads to better \emph{propagation tightness}, while non-IBP-based methods do not. Additionally, they note that over-regularisation provides worse standard-certified accuracy trade-offs.
In line with these findings, all current state-of-the-art certified training methods \cite{muller2023certifiedtrainingsmallboxes, mao2023, palma2024, depalma2026} move away from previous over-regularised approaches \cite{gowal2018,balunovic2020colt,palma2022} and instead shift their focus to combining AT with IBP-based training, achieving state-of-the-art standard-certified accuracy trade-offs while preserving certification properties. 
\citet{palma2024} further elucidate this in their definition of \emph{expressive losses} whereby expressivity is attained through interpolation between lower and upper bounds of Equation \eqref{eq:rob-inequality} during training, providing an unsound but favourable solution to Equation \eqref{eq:train_rob}.

\subsection{Knowledge Distillation}
\label{sub-sec:knowledge-dist}
\citet{hinton2015} introduced knowledge distillation as a method for transferring knowledge from a large teacher model, $T_{\theta_t}$, to a smaller, more compact student model, $S_{\theta_s}$. This allows the student to retain much of the teacher's performance while reducing computational cost. Knowledge distillation utilises Kullback-Leibler (KL) divergence to compare the predictive distributions of the student and teacher, which is defined by: 
\begin{equation}
    \mathrm{KL}\!\left(P^\tau_{\theta_t}(\mathbf{x}) \,\|\, P_{\theta_s}^\tau(\mathbf{x})\right) = \sum^C_{i=1} [ P_{\theta_t}^\tau(\mathbf{x})]_i \log \left( \frac{ [ P_{\theta_t}^\tau(\mathbf{x})]_i}{ [ P_{\theta_s}^\tau(\mathbf{x})]_i} \right),
    \label{eq:KL}
\end{equation}
where $\tau$ is a temperature applied to the softmax operation and $P_{\theta_t}^\tau(\mathbf{x}) := \mathrm{softmax}\left(\frac{T_{\theta_t}(\mathbf{x})}{\tau}\right)$, $P_{\theta_s}^\tau(\mathbf{x}) := \mathrm{softmax}\left(\frac{S_{\theta_s}(\mathbf{x})}{\tau}\right)$ are the respective predictive distributions of the teacher and student after softmax. Knowledge distillation then optimises the following objective:
\begin{equation}
    \min_{\theta_s} \mathbb{E}_{(\mathbf{x}, y) \sim \mathcal{D}} \left[ (1-\alpha)\mathcal{L}_{\mathrm{CE}} \left(S_{\theta_s} ( \mathbf{x}), y\right) 
    +
    \alpha \tau^2 \mathrm{KL}\!\left(P_{\theta_t}^\tau(\mathbf{x}) \,\|\, P_{\theta_s}^\tau(\mathbf{x})\right)
    \right],
    \label{eq:KD}
\end{equation}
where $\alpha\in[0,1]$ is a balancing factor between standard training and learning the teacher's predictive distribution. Knowledge distillation has since been applied in the AT setting to boost empirical robustness of a student using an adversarially trained teacher \cite{goldblum2020, zi2021, zhu2022, cui2024}. This was first explored by \citet{goldblum2020}, through adversarially robust distillation (ARD), which injects adversarial examples into the KL term from Equation \eqref{eq:KD}. More specifically:
\begin{equation}
    \mathcal{L}_{\mathrm{ARD}} := (1-\alpha)\mathcal{L}_{\mathrm{CE}} \left(S_{\theta_s} ( \mathbf{x}), y\right) 
    +
    \alpha \tau^2 \mathrm{KL}\!\left(P_{\theta_t}^\tau(\mathbf{x}) \,\|\, P_{\theta_s}^\tau(\mathbf{x_{adv}})\right).
    \label{eq:ARD}
\end{equation}
Additionally, \citet{depalma2026} recently introduced knowledge distillation within the field of certified training, where knowledge is distilled from an empirically robust teacher over the feature space (further details in Appendix \ref{app:cc-dist}).
In the coming section, we build on previous certified training \cite{gowal2018, balunovic2020colt, palma2022, mao2023, muller2023certifiedtrainingsmallboxes, palma2024, depalma2026} and adversarial knowledge distillation \cite{goldblum2020, zi2021} works, presenting a novel certified training objective that distils the predictive distribution of an empirically robust teacher onto a student training for certifiable robustness.

\section{Methodology}
\label{sec:methodology}
In this section, we introduce \emph{Adversarial Distillation for Certification} (AD-CERT), a novel certified training scheme that interpolates between adversarial distillation (AD) and IBP \cite{gowal2018}.
We leverage the key insight that replacing a hard-label adversarial lower bound with a teacher-guided surrogate could further improve certified training. In this way, the empirical teacher shapes the student on concrete adversarial examples, while certification is enforced directly through IBP. Full proofs and additional theoretical results are deferred to Appendices \ref{app:proofs} and \ref{app:additional_results}, respectively. Pseudocode for the complete training procedure is provided in Appendix \ref{app:pseudo} and further details on AD-CERT's computational cost can be found in Appendix \ref{app:complex}, where AD-CERT corresponds to performing IBP and PGD training, with the slight added cost of a single forward pass on the detached teacher.

\subsection{AD-CERT Formulation}
\label{sub-sec:ad-cert}
Going forward, we denote by $f_{\theta_t}$ a fixed teacher network trained via AT, and let $f_{\theta_s}$ be the student model that we aim to certify. Note that both the student and teacher share the same architecture in this distillation setting and adversarial examples ($\mathbf{x_{adv}}$) are computed by running the PGD attack against the student's parameters. Moreover, for all experiments, we use a temperature setting of $\tau = 1$, and henceforth, omit the $\tau$ notation from all softmax operations, writing $P_\theta(\mathbf{x}) := \mathrm{softmax}(f_\theta(\mathbf{x}))$ for the predictive distribution of a network.
Excluding additional regularisation, AD-CERT attains the following loss objective:
\begin{equation}
    \mathcal{L}_{\mathrm{AD\mbox{-}CERT}} := (1-\alpha)\underbrace{{\mathrm{KL}} \left(P_{\theta_t}(\mathbf{x}) \,\|\, P_{\theta_s}(\mathbf{x_{adv}}) \right)}_{\text{Adversarial Distillation}} + 
     \alpha \underbrace{\mathcal{L}_\mathrm{CE}\left(-\mathbf{\underline{z}} _{\theta_s}^\Delta(\mathbf{x}, y), y \right)}_{\text{IBP}}, \quad \alpha \in [0, 1].
    \label{eq:ad_cert_loss}
\end{equation}
The AD component of Equation \eqref{eq:ad_cert_loss} is exactly the KL portion of Equation \eqref{eq:ARD}, where the student is evaluated on an adversarially perturbed input, and tasked to match the clean distribution of an adversarially trained teacher, aiming to preserve standard and empirically robust accuracy. 
On an intuitive level, AD-CERT is designed to keep the teacher signal completely outside of verified bounds, using it purely for its empirical strengths, while leveraging IBP to tighten the worst-case logit differences. This aims to make the student amenable to verification, while still retaining a large proportion of the teacher's adversarial and standard accuracy.

\subsection{Teacher-Guided Lower Bound Surrogates}
\label{sub-sec:lower-bound-surrogate}
In this subsection, we take an analytical view of the objective in Equation \eqref{eq:ad_cert_loss} and formalise the role of AD within certified training. Unlike AT and IBP, which, respectively, provide formal lower and upper bounds from Equation \eqref{eq:rob-inequality}, the AD term in AD-CERT is used as a teacher-guided surrogate for the adversarial lower bound. This is an intentional relaxation of the lower bound objective, using an empirical teacher only to improve the adversarial endpoint, while leaving the verifiable IBP upper bound objective unchanged. We begin by showing that, up to a teacher-dependent constant, the distillation component of AD-CERT is exactly AT with soft teacher labels.
\begin{proposition}[Equivalence of AD and soft-label AT]
Let $\mathbf{x_{adv}}$ be fixed, and let the teacher predictive distribution $P_{\theta_t}(\mathbf{x})$ be fixed. Then the AD objective:
\begin{equation}
    \mathcal{L}_{\mathrm{AD}}
    :=
    \mathrm{KL}\!\left(P_{\theta_t}(\mathbf{x}) \,\|\, P_{\theta_s}(\mathbf{x_{adv}})\right),
\end{equation}
is equal, up to an additive constant independent of the student parameters, to AT with soft teacher labels, namely:
\begin{equation}
    \mathcal{L}_{\mathrm{soft\mbox{-}AT}}
    :=
    \mathcal{L}_\mathrm{CE}\!\left(f_{\theta_s}(\mathbf{x_{adv}}),\, P_{\theta_t}(\mathbf{x})\right).
\end{equation}
Moreover, both objectives induce identical backpropagation updates with respect to the student logits. In particular, for each class index $i$:
\begin{equation}
    \frac{\partial \mathcal{L}_{\mathrm{AD}}}{\partial z_{\theta_s}^i}
    =
    [P_{\theta_s}(\mathbf{x_{adv}})]_i - [P_{\theta_t}(\mathbf{x})]_i
    =
    \frac{\partial \mathcal{L}_{\mathrm{soft\mbox{-}AT}}}{\partial z_{\theta_s}^i}.
\end{equation}
\label{prop:ad-cert-soft-label}
\end{proposition}
\begin{proof}
    See Appendix \ref{app:proofs}.
\end{proof}
\noindent
Given Proposition \ref{prop:ad-cert-soft-label}, we have shown that the distillation component of AD-CERT preserves AT at the gradient level, while being conditioned by an empirical teacher. We further elucidate this connection through an analytical view of our loss function from Equation \eqref{eq:ad_cert_loss}.
\begin{proposition}[Analytical form of AD-CERT]
Let $H(P)$ and $H(Q,P)$ denote the entropy and cross entropy respectively for probability distributions $P$ and $Q$. That is:
\begin{equation}
    H(P) =  -\sum_{i=1}^C p_i\log(p_i), \qquad H(Q,P) = -\sum_{i=1}^C q_i\log(p_i).
\end{equation}
Then, $\mathcal{L}_{\mathrm{AD\mbox{-}CERT}}$ from Equation \eqref{eq:ad_cert_loss} can be re-written as:
\begin{equation}
    \mathcal{L}_{\mathrm{AD\mbox{-}CERT}}
    =
    (1-\alpha)H(P_{\theta_t}(\mathbf{x}),P_{\theta_s}(\mathbf{x_{adv}}))
    +
    \alpha \mathcal{L}_{\mathrm{IBP}}
    -
    (1-\alpha)H(P_{\theta_t}(\mathbf{x})).
\end{equation}
Therefore, up to a constant independent of $\theta_s$, AD-CERT is a scalarisation between
(i) a soft-label adversarial lower bound surrogate and
(ii) a certified IBP upper bound.
\label{prop:ad-analytical}
\end{proposition}
\begin{proof}
    See Appendix \ref{app:proofs}.
\end{proof}
\noindent
Given Proposition \ref{prop:ad-analytical}, AD-CERT fits naturally within the broader class of modern certified training objectives \cite{muller2023certifiedtrainingsmallboxes, mao2023, palma2024, depalma2026} that combine empirical and certifiable endpoints. Its distinguishing feature is that the adversarial endpoint is no longer a formal lower bound of Equation \eqref{eq:train_rob}, but a relaxed surrogate induced by AD from an empirically robust teacher. Additionally, this positions AD-CERT as distinct and complementary to the recent work of CC-DIST \cite{depalma2026}, which computes both IBP bounds and adversarial examples within their robust feature-space distillation objective. Instead, AD-CERT integrates teacher distillation directly into the lower bound endpoint, while keeping IBP bounds purely for the certified endpoint.
\\ \\
While Proposition \ref{prop:ad-analytical} characterises AD-CERT at the level of its objective, the next result highlights a key geometric difference between soft-label and hard-label adversarial supervision. In particular, the AD branch of AD-CERT induces a finite target in the logit space, whereas standard AT only approaches its optimum as the correct class margin grows indefinitely.
\begin{lemma}
For AD, the optimal logit margin, $\Delta z_{ij}^*$, between any two classes $i,j$ is:
\begin{equation}
\Delta z_{ij}^* = \log\left(\frac{[P_{\theta_t}(\mathbf{x})]_i}{[P_{\theta_t}(\mathbf{x})]_j}\right).
\end{equation}
Moreover, standard AT has no finite minimiser, and its infimum $0$ is approached only as $[f_{\theta_s}(\mathbf{x_{adv}})]_y-[f_{\theta_s}(\mathbf{x_{adv}})]_j\to+\infty, \quad \forall j\neq y$.
\label{lemma:logit-margin}
\end{lemma}
\begin{proof}
    See Appendix \ref{app:proofs}.
\end{proof}
\noindent
Lemma \ref{lemma:logit-margin} suggests that teacher-guided soft-label supervision may be a principled approach to enhance certified training, since it arguably induces a smoother lower bound objective, not pushing for arbitrarily large class margins. We evaluate our hypotheses experimentally in the next section.
\section{Experimental Evaluation}
\label{sec:evaluation}
In this section, we conduct experiments to assess the effectiveness of our proposed certified training method, AD-CERT (§\ref{sub-sec:ad-cert}). Across our experiments, we report standard, certified and \textsc{AutoAttack} (AA) \cite{croce2020} accuracy. Certified accuracy is the primary metric used in certified training literature \cite{gowal2018, balunovic2020colt, palma2022, muller2023certifiedtrainingsmallboxes, mao2023, palma2024, mao2025, depalma2026}, defined as the proportion of the validation set for which a verifier proves robustness within  $\mathcal{B}_\varepsilon(\mathbf{x})$. Standard accuracy is the clean accuracy on the validation set without perturbation. AA accuracy refers to the proportion of the validation set that could not be attacked by the AA strategy. As such, AA accuracy is strictly greater than or equal to certified accuracy. Additionally, we provide supplementary experiments that are complementary to those outlined in this section in Appendix \ref{app:additional-experiments}. 
\subsection{Settings}
We implement AD-CERT on top of CTBench \cite{mao2025}, a unified certified training library using PyTorch \cite{paszke2019}. We adopt three datasets, MNIST \cite{lecun2010mnist}, CIFAR-10 \cite{krizhevsky2009} and TinyImageNet \cite{le2015tiny}. We train on the corresponding training sets and certify on the validation sets. To align with recent literature \cite{shi2021fastcertifiedrobusttraining, muller2023certifiedtrainingsmallboxes, mao2023, palma2024, mao2025, depalma2026}, we use a 7-layer convolutional network ($\mathrm{CNN7}$) for both the teacher and student models, and the perturbation radii applied are $\varepsilon = \{0.1, 0.3\}$ for MNIST, $\varepsilon = \{\frac{2}{255}, \frac{8}{255}\}$ for CIFAR-10 and $\varepsilon = \{\frac{1}{255}\}$ for TinyImageNet. For verification, we use the $\alpha$,$\beta$-CROWN library \cite{zhang2018,xu2020,xu2021fast,wang2021beta}, which combines linear bound propagation with branch-and-bound \cite{bunel2020branchboundpiecewiselinear} for complete neural network verification. Full details on training of both student and teacher models, certification and the experimental setup can be found in Appendix \ref{app:exp-details}.
\subsection{Comparison with State-of-the-Art}
\label{sub-sec:main-results}
Table \ref{tab:evaluation_results} compares the performance of AD-CERT (§\ref{sub-sec:ad-cert}) to current state-of-the-art certified training methods \cite{gowal2018, zhang2022, muller2023certifiedtrainingsmallboxes, mao2023, palma2024, depalma2026}. For each method, we report the standard and certified accuracy on the validation set. When available, we use the benchmark results from CTBench \cite{mao2025}, as they generally report the strongest certified accuracies of each reported certified training method compared to the original paper \cite{gowal2018, zhang2022, muller2023certifiedtrainingsmallboxes, mao2023, palma2024, depalma2026}, while also providing a fair comparison to AD-CERT, which is built on top of the unified library. \\ \\
Overall, AD-CERT achieves state-of-the-art certified accuracy for ReLU-based architectures across all evaluated settings, while retaining competitive standard accuracy. On MNIST, AD-CERT obtains modest but consistent improvements over prior methods at both perturbation radii. While the gains are small, this is expected given the relative simplicity of MNIST and the already saturated performance of existing methods. For CIFAR-10 and TinyImageNet, AD-CERT improves the best prior ReLU-based certified accuracy by an average of $0.37$ percentage points, with the largest gain observed at $\varepsilon=\frac{2}{255}$ on CIFAR-10. 
For CIFAR-10 at $\varepsilon=\frac{8}{255}$, 
SORTNET \cite{zhang2022} obtains higher certified accuracy, however, it relies on a specialised 1-Lipschitz architecture known to perform particularly well on this setting. We therefore distinguish it from the ReLU-based methods, with further discussion in related work (§\ref{sec:related_work}). On TinyImageNet, AD-CERT achieves the highest certified accuracy among all compared methods, while substantially improving standard accuracy compared to previous non-distillation-based certified training methods \cite{gowal2018, muller2023certifiedtrainingsmallboxes, mao2023, palma2024}. This suggests that AD-CERT scales favourably beyond the smaller MNIST and CIFAR-10 settings, where certified training is often closer in performance. An interesting comparison is that of AD-CERT and CC-DIST \cite{depalma2026}, the only distillation-based certified training methods. In general, CC-DIST retains higher standard accuracy, while AD-CERT obtains higher certified accuracy. We investigate this trade-off further through a controlled comparison in Table \ref{tab:cert_distillation_results}.
\begin{table}[h!]
\centering
\caption{\small Comparison of AD-CERT with prior state-of-the-art certified training methods. Bold indicates the best standard or certified accuracy on each benchmark. When the best ReLU-based architecture result differs from the best overall result, the top ReLU-based accuracy is bold-underlined.}
\label{tab:evaluation_results}
\footnotesize
\begin{tabular}{l c l c c c}
\toprule
Dataset & $\varepsilon$ & Method & Source & Std. acc. [\%] & Cert. acc. [\%] \\ \midrule
\multirow{14}{*}{MNIST} 
& \multirow{6}{*}{$0.1$} 
& AD-CERT & Ours & \textbf{99.25} & \textbf{98.53} \\ \cmidrule{3-6}
& & IBP & \citet{mao2025} & 98.87 & 98.26 \\
& & SORTNET$^\dagger$ & \citet{zhang2022} & 99.01 & 98.14 \\
& & SABR & \citet{mao2025} & 99.08 & 98.40 \\
& & TAPS/STAPS$^\ddagger$ & \citet{mao2025} & 99.16 & 98.52 \\
& & MTL-IBP & \citet{mao2025} & 99.18 & 98.37
\\ \cmidrule{2-6}
& \multirow{6}{*}{$0.3$} 
& AD-CERT & Ours & \textbf{98.85} & \textbf{93.98} \\ 
\cmidrule{3-6}
& & IBP & \citet{mao2025} & 98.54 & 93.80 \\
& & SORTNET$^\dagger$ & \citet{zhang2022} & 98.46 & 93.40 \\
& & SABR & \citet{mao2025} & 98.66 & 93.68 \\
& & TAPS/STAPS$^\ddagger$ & \citet{mao2025} & 98.56 & 93.95 \\
& & MTL-IBP & \citet{mao2025} & 98.74 & 93.90
\\ \midrule
\multirow{16}{*}{CIFAR-10} 
& \multirow{7}{*}{$\frac{2}{255}$} 
& AD-CERT & Ours & 79.21 & \textbf{65.08} \\ \cmidrule{3-6}
& & IBP & \citet{mao2025} & 67.49 & 55.99 \\
& & SORTNET$^\dagger$ & \citet{zhang2022} & 67.72 & 56.94 \\
& & SABR & \citet{mao2025} & 77.86 & 63.61 \\
& & TAPS/STAPS$^\ddagger$ & \citet{mao2025} & 77.05 & 64.21 \\
& & MTL-IBP & \citet{mao2025} & 78.82 & 64.41 \\
& & CC-DIST & \citet{depalma2026} & \textbf{81.55} & 64.60
\\ \cmidrule{2-6}
& \multirow{7}{*}{$\frac{8}{255}$} 
& AD-CERT & Ours & 53.69 & \textbf{\underline{35.84}} \\ \cmidrule{3-6}
& & IBP & \citet{mao2025} & 48.51 & 35.28 \\
& & SORTNET$^\dagger$ & \citet{zhang2022} & 54.84 & \textbf{40.39} \\
& & SABR & \citet{mao2025} & 52.71 & 35.34 \\
& & TAPS/STAPS$^\ddagger$ & \citet{mao2025} & 49.96 & 35.25 \\
& & MTL-IBP & \citet{mao2025} & 54.28 & 35.41 \\
& & CC-DIST & \citet{depalma2026} & \textbf{55.13} & 35.52
\\ \midrule
\multirow{8}{*}{TinyImageNet} 
& \multirow{8}{*}{$\frac{1}{255}$} 
& AD-CERT & Ours & 40.06 & \textbf{28.19} \\ \cmidrule{3-6}
& & IBP & \citet{mao2025} & 26.77 & 19.82 \\
& & SORTNET$^\dagger$ & \citet{zhang2022} & 25.69 & 18.18 \\
& & SABR & \citet{mao2025} & 30.58 & 20.96 \\
& & TAPS/STAPS$^\ddagger$ & \citet{mao2025} & 30.63 & 22.31 \\
& & MTL-IBP & \citet{mao2025} & 35.97 & 27.73 \\ 
& & CC-DIST & \citet{depalma2026} & \textbf{43.78} & 27.88
\\ \bottomrule
\addlinespace
\multicolumn{6}{l}{\footnotesize 
    \parbox{0.95\textwidth}{
        $^\dagger$ Non-ReLU-based architecture. \\
        $^\ddagger$ We report the stronger of TAPS and STAPS based on certified accuracy.
    }
}
\end{tabular}
\vspace{-0.5cm}
\end{table}
\paragraph{Comparison of Distillation-based Certified Training Methods}
Table \ref{tab:cert_distillation_results} compares AD-CERT against alternative distillation-based certified training methods under the exact same experimental setup for a direct comparison. This controlled comparison allows us to isolate the effect of the distillation mechanism itself by training and evaluating all methods in the CTBench library \cite{mao2025} with the same teacher model, training pipeline and verifier. We compare AD-CERT with CC-DIST \cite{depalma2026}, which performs robust feature-space distillation. Additionally, we introduce RSLAD-CERT, an alternative logit-level objective obtained by replacing the adversarial distillation term in AD-CERT with RSLAD \cite{zi2021}, but retaining the same IBP interpolation, \ie $L_{\mathrm{RSLAD\mbox{-}CERT}}:= (1-\alpha)L_{\mathrm{RSLAD}} + \alpha L_{\mathrm{IBP}}$. For all three methods, we use the exact same hyperparameter values as those used for AD-CERT in Table \ref{tab:evaluation_results}, which we report in detail in Appendix \ref{sub-sec:train_det}, Table \ref{tab:adcert_hyperparameters}. The only exception is $w_{rob}$ for TinyImageNet, where we used an earlier untuned value of $0.7$, as per MTL-IBP's \cite{palma2024} reported value in CTBench \cite{mao2025}.
\\ \\ 
Overall, the results support the design choice of AD-CERT, showing that directly distilling the teacher's predictive distribution at the adversarial endpoint provides a simple and effective empirical branch for certified training. The addition of the clean distillation term in RSLAD-CERT reduces certifiability by $0.6$ percentage points on average, while only improving standard accuracy by an average of $0.38$ percentage points compared to AD-CERT. Consistent with the pattern observed in Table \ref{tab:evaluation_results}, CC-DIST is generally strongest at transferring standard accuracy, but shows weaker certified accuracy under the controlled setup. This results in worse standard-certified accuracy trade-offs, with CC-DIST improving standard accuracy over AD-CERT by an average of $1.69$ percentage points, but reducing certifiability by $3.52$ percentage points. Both logit-level objectives, \ie RSLAD-CERT and AD-CERT, improve certified accuracy over the feature-space distillation objective of CC-DIST, despite not matching the formal definition of expressivity \cite{palma2024}. One possible explanation is that CC-DIST induces a more regularised objective, since it computes adversarial examples and IBP bounds during distillation, before being coupled with another expressive loss CC-IBP \cite{palma2024}. This suggests that tightly coupling distillation-based certified training objectives with the formal definition of expressivity may not translate into stronger certification properties. In contrast, AD-CERT keeps the distillation branch as a pure adversarial endpoint, while leaving the certified component as the standard IBP loss. This relaxes the definition of expressivity by extending it to surrogate lower bounds (§\ref{sub-sec:lower-bound-surrogate}), prioritising simplicity and smoothness in the loss objective. Finally, we acknowledge that further exhaustive hyperparameter and teacher tuning on CC-DIST could potentially reduce the gap between the results reported in Table \ref{tab:cert_distillation_results} and those in Table \ref{tab:evaluation_results}. However, the most important shared hyperparameters, namely the IBP coefficient $\alpha$ and $L_1$ regularisation, coincide with the original CC-DIST setup \cite{depalma2026} (Table 3). The main remaining difference is therefore the teacher model. For completeness, we include an additional ablation on the sensitivity of AD-CERT and CC-DIST to different teachers in Appendix \ref{app:additional-experiments}, which shows a similar trend to the above ablation.
\begin{table}[!h]
\centering
\caption{Comparison of AD-CERT against distillation-based certified training methods under the same experimental setup.}
\label{tab:cert_distillation_results}
\footnotesize
\begin{tabular}{l c l c c c c c}
\toprule
\multicolumn{1}{c}{Dataset} 
& $\varepsilon$ 
& \multicolumn{1}{c}{Method} 
& \multicolumn{1}{c}{Method source$^\dagger$} 
& Dist. space 
& Std. [\%] 
& AA [\%] 
& Cert. [\%] \\ 
\cmidrule{1-8}
\multirow{6}{*}{CIFAR-10} 
& \multirow{3}{*}{$\frac{2}{255}$} 
& AD-CERT & Ours & Logits & 79.21 & 68.38 & \textbf{65.20} \\
& & RSLAD-CERT & Ours & Logits & 79.87 & 68.27 & 64.36 \\
& & CC-DIST & \citet{depalma2026} & Features & \textbf{81.04} & \textbf{69.97} & 59.80 \\ 
\cmidrule{2-8}
& \multirow{3}{*}{$\frac{8}{255}$} 
& AD-CERT & Ours & Logits & 53.49 & 36.46 & \textbf{35.60} \\
& & RSLAD-CERT & Ours & Logits & 53.95 & \textbf{36.50} & 35.19 \\
& & CC-DIST & \citet{depalma2026} & Features & \textbf{54.32} & 35.69 & 34.55 \\ 
\midrule
\multirow{3}{*}{TinyImageNet} 
& \multirow{3}{*}{$\frac{1}{255}$} 
& AD-CERT & Ours & Logits & 41.38 & 30.68 & \textbf{27.68} \\
& & RSLAD-CERT & Ours & Logits & 41.39 & 30.90 & 27.14 \\
& & CC-DIST & \citet{depalma2026} & Features & \textbf{43.79} & \textbf{32.66} & 23.57 \\ 
\bottomrule
\end{tabular}
\begin{minipage}{\linewidth}
\footnotesize $^\dagger$ The ``Method source'' column indicates the origin of the training objective, all results are from our implementation and evaluation in the CTBench library \cite{mao2025}.
\end{minipage}
\end{table}
\subsection{Ablations \& Further Discussions}
\paragraph{Teacher-Student Comparison}
\label{sub-sec:teacher-student-discussion}
Table \ref{tab:student_teacher_comparison} compares each AD-CERT student with its corresponding empirically robust teacher. Overall, AD-CERT converts models that are empirically robust but not certifiable into students with strong certified accuracy, while retaining a substantial fraction of the teacher's standard and adversarial performance. This effect is particularly clear on MNIST, where AD-CERT stays within one percentage point of the teacher in both standard and AA accuracy across both perturbation radii, while introducing strong certifiability. On TinyImageNet, the teacher-student gap remains moderate, with losses below seven percentage points in standard accuracy and close to two percentage points in AA accuracy, though notably achieving $28.19\%$ in certified accuracy. The largest gap occurs on CIFAR-10 at $\varepsilon=\frac{8}{255}$, where the student sees a decrease of almost twenty percentage points in standard accuracy relative to the teacher, although the drop in AA accuracy is much smaller. Nevertheless, AD-CERT still achieves the strongest certified accuracy among ReLU-based architectures in this setting. However, this highlights that closing the remaining gap between empirically robust teachers and certifiably robust students at larger perturbation radii remains an important direction for future work. Motivated by this observation, we next ablate different pure AD techniques on CIFAR-10, aiming to better understand the difficulty of transferring empirical robustness into certifiable robustness and to identify potential directions for stronger distillation-based certified training objectives in future.
\begin{table}[h]
\centering
\caption{Comparison between teacher and student models of AD-CERT for each dataset entry in Table \ref{tab:evaluation_results}. Deltas are reported relative to the teacher at the same perturbation radius.}
\label{tab:student_teacher_comparison}
\footnotesize
\begin{tabular}{l c l c c c c c}
\toprule
\multicolumn{1}{c}{Dataset} 
& $\varepsilon$ 
& \multicolumn{1}{c}{Model} 
& Std. acc. [\%] 
& $\Delta$ Std. 
& AA acc. [\%] 
& $\Delta$ AA 
& Cert. acc. [\%] \\
\cmidrule{1-8}
\multirow{4}{*}{MNIST} 
& \multirow{2}{*}{$0.1$} 
& Teacher & 99.47 & -- & 98.97 & -- & $\approx0^\dagger$ \\
& & Student & 99.25 & \negv{-0.22} & 98.55 & \negv{-0.42} & 98.53 \\
\cmidrule{2-8}
& \multirow{2}{*}{$0.3$} 
& Teacher & 99.43 & -- & 93.83 & -- & $\approx0^\dagger$ \\
& & Student & 98.85 & \negv{-0.58} & 94.64 & \pos{0.81} & 93.98 \\
\midrule
\multirow{4}{*}{CIFAR-10} 
& \multirow{2}{*}{$\frac{2}{255}$} 
& Teacher & 88.67 & -- & 72.41 & -- & $\approx0^\dagger$ \\
& & Student & 79.21 & \negv{-9.46} & 68.39 & \negv{-4.02} & 65.08 \\
\cmidrule{2-8}
& \multirow{2}{*}{$\frac{8}{255}$} 
& Teacher & 73.55 & -- & 41.99 & -- & $\approx0^\dagger$ \\
& & Student & 53.69 & \negv{-19.86} & 36.69 & \negv{-5.30} & 35.84 \\
\midrule
\multirow{2}{*}{TinyImageNet} 
& \multirow{2}{*}{$\frac{1}{255}$} 
& Teacher & 46.78 & -- & 33.16 & -- & $\approx0^\dagger$ \\
& & Student & 40.06 & \negv{-6.72} & 31.14 & \negv{-2.02} & 28.19 \\
\bottomrule
\multicolumn{8}{l}{\footnotesize $^\dagger$ None of the first 10 samples were verified within a 1000 second time-limit.}
\end{tabular}
\end{table}
\paragraph{Exploring Adversarial Distillation Techniques}
Table \ref{tab:dist_results} compares several distillation objectives on CIFAR-10 \cite{krizhevsky2009} after removing any certified training component and trained against a PGD-trained teacher on the same architecture. This ablation isolates the empirical distillation endpoint, allowing us to study which teacher-student objectives best preserve the robustness of the PGD teacher before introducing certifiability. Natural knowledge distillation \cite{hinton2015}, denoted Nat-KL, transfers the teacher's clean predictive distribution to the student. AD corresponds to removing the IBP term from AD-CERT (§\ref{sub-sec:ad-cert}), while RSLAD \cite{zi2021} extends adversarially robust distillation \cite{goldblum2020} using soft-labels. Finally, adversarial feature-space knowledge distillation (AF-KD) corresponds to setting $\alpha=0$ in CC-DIST \cite{depalma2026}, thereby removing the IBP components from the objective. For RSLAD and AF-KD, we use $\lambda=\frac{5}{6}$ and $\beta=\frac{5}{k}$, respectively, following the suggestions of the original works \cite{zi2021, depalma2026}, where $k$ denotes the dimensionality of the feature space.  \\ \\
From the results, it can be observed that Nat-KL improves standard accuracy at both perturbation radii, but dramatically fails to transfer adversarial robustness, losing more than $41$ percentage points in AA accuracy in both settings. This confirms that matching an adversarially trained teacher only on clean inputs is insufficient for transferring adversarial robustness. In contrast, objectives that expose the student to adversarial inputs preserve teacher-level AA accuracy much more effectively. Across both perturbation radii, AD and RSLAD remain close to the teacher's AA accuracy, while AF-KD is competitive at the smaller radius but degrades more noticeably at $\varepsilon=\frac{8}{255}$. Among these methods, AD provides the most direct and consistently effective adversarial endpoint, improving teacher-level AA accuracy at the smaller radius and remaining within one percentage point at the larger radius. These results further motivate our use of AD as the empirical endpoint of AD-CERT, since it provides a simple but effective mechanism for retaining adversarial accuracy from an empirically robust teacher, while remaining straightforward to combine with an IBP-certified upper bound. This yields a loss objective that resembles previous state-of-the-art expressive losses \cite{palma2024}, but in a distillation setting with lower bound surrogates, as discussed in §\ref{sub-sec:lower-bound-surrogate}.
\begin{table}[h]
\centering
\caption{Effect of different AD techniques without a certification component on CIFAR-10.}
\vspace{0.25em}
\label{tab:dist_results}
\footnotesize
\begin{minipage}{0.9\linewidth}
\centering
\begin{tabular}{@{}l l@{}}
\toprule
Method & Objective \\
\midrule
Nat-KL & $\mathrm{KL}\!\left(P_{\theta_t}(\mathbf{x}) \,\|\, P_{\theta_s}(\mathbf{x})\right)$ \\
AD$^\dagger$ & $\mathrm{KL}\!\left(P_{\theta_t}(\mathbf{x}) \,\|\, P_{\theta_s}(\mathbf{x_{adv}})\right)$ \\
RSLAD$^\ddagger$ & $(1-\lambda)\mathrm{KL}\!\left(P_{\theta_t}(\mathbf{x}) \,\|\, P_{\theta_s}(\mathbf{x})\right) + \lambda\mathrm{KL}\!\left(P_{\theta_t}(\mathbf{x}) \,\|\, P_{\theta_s}(\mathbf{x_{adv}})\right)$ \\
AF-KD$^\dagger$ & $\mathcal{L}_{\mathrm{CE}}\!\left(-\mathbf{z}_{\theta_s}(\mathbf{x_{adv}},y), y\right) + \beta \left\|h_{\theta_s}(\mathbf{x_{adv}}) - h_{\theta_t}(\mathbf{x})\right\|_2^2$ \\
\bottomrule
\end{tabular}
\vspace{0.25em}
\end{minipage}
\vspace{0.6em}
\begin{tabular}{l c l c c c c}
\toprule
\multicolumn{1}{c}{Dataset} & $\varepsilon$ & \multicolumn{1}{c}{Method} 
& Std. acc. [\%] & $\Delta$ Std. 
& AA acc. [\%] & $\Delta$ AA \\
\midrule
\multirow{10}{*}{CIFAR-10}
& \multirow{5}{*}{$\frac{2}{255}$}
& PGD Teacher & 88.67 & -- & 72.41 & -- \\
& & Nat-KL    & \textbf{89.89} & \pos{1.22} & 30.72 & \negv{-41.69} \\
& & AD$^\dagger$    & 88.64 & \negv{-0.03} & \textbf{73.62} & \pos{1.21} \\
& & RSLAD$^\ddagger$ & 87.88 & \negv{-0.79} & 73.32 & \pos{0.91} \\
& & AF-KD$^\dagger$   & 89.52 & \pos{0.85} & 73.55 & \pos{1.14} \\
\cmidrule{2-7}
& \multirow{5}{*}{$\frac{8}{255}$}
& PGD Teacher & 78.95 & -- & \textbf{42.48} & -- \\
& & Nat-KL    & \textbf{81.32} & \pos{2.37} & 0.52  & \negv{-41.96} \\
& & AD$^\dagger$    & 77.08 & \negv{-1.87} & 41.64 & \negv{-0.84} \\
& & RSLAD$^\ddagger$ & 76.47 & \negv{-2.48} & 42.38 & \negv{-0.10} \\
& & AF-KD$^\dagger$   & 78.95 & \neutral{0.00} & 37.49 & \negv{-4.99} \\
\bottomrule
\end{tabular}
\begin{minipage}{0.8\linewidth}
\footnotesize
$^\dagger$ IBP terms removed from AD-CERT/CC-DIST, \ie $\alpha=0$. \\
$^\ddagger$ Adversarial examples are generated against soft-labels as in original work \cite{zi2021}.
\end{minipage}
\end{table}
\paragraph{IBP Coefficient Sensitivity Analysis}
In certified training, it is well known that higher certified accuracy often comes at the cost of standard accuracy. Recent methods fall into a family of expressive losses \cite{palma2024}, \eg SABR \cite{muller2023certifiedtrainingsmallboxes}, CC-IBP \& MTL-IBP \cite{palma2024}, and CC-DIST \cite{depalma2026}, including hyperparameters that regulate this trade-off by interpolating between empirical and certifiable training objectives. Our goal in this study is to empirically support our methodology claim (§\ref{sub-sec:lower-bound-surrogate}) that the AD term in AD-CERT provides a sensible teacher-guided surrogate for the lower bound endpoint used in prior expressive losses \cite{muller2023certifiedtrainingsmallboxes, palma2024, depalma2026}.\\ \\
Figure \ref{fig:alpha_sensitivity_tinyimagenet} shows the trade-off between standard, AA, and certified accuracy for AD-CERT on TinyImageNet with $\varepsilon=\frac{1}{255}$ as we decrease the IBP coefficient, $\alpha$, from Equation \ref{eq:ad_cert_loss}. Consistent with the sensitivity analyses of \citet{muller2023certifiedtrainingsmallboxes} (Figure 7), \citet{palma2024} (Figure 1) and \citet{mao2025} (Figure 7), standard and AA accuracy increase as the strength of IBP regularisation is reduced. Up to a point, this also benefits certified accuracy, before the model starts to dramatically lose its certification properties when $\alpha$ is decreased beyond this. This behaviour suggests that the adversarial distillation endpoint in AD-CERT acts analogously to the hard-label adversarial endpoints used in prior expressive losses, while also achieving state-of-the-art certified accuracy (§\ref{sub-sec:main-results}).
\vspace{-0.2cm}
\begin{figure}[h]
    \centering
    \includegraphics[width=0.70\linewidth]{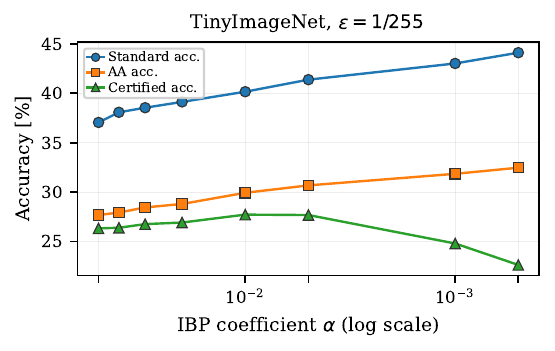}
    \vspace{-0.3cm}
    \caption{
        Sensitivity analysis of AD-CERT with respect to the IBP coefficient ($\alpha$).
        We report standard, AA, and certified accuracy over eight values of $\alpha \in [5\times10^{-4}, 5\times10^{-2}]$.
    }
    \vspace{-0.75cm}
    \label{fig:alpha_sensitivity_tinyimagenet}
\end{figure}

\section{Related Work}
\label{sec:related_work}
\paragraph{Certified Training} As previously mentioned (§\ref{sub-sec:train-rob}), IBP \cite{mirman2018, gowal2018} performs interval arithmetic over the bounds of a simple bounding-box relaxation that over-approximates the possible output ranges of each layer, before computing an upper bound of the worst-case loss based on the output range of the final layer. \citet{shi2021fastcertifiedrobusttraining} later improved IBP-based training by introducing a parameter initialisation technique that induces a constant growth rate of IBP bounds, along with specially designed regularisation to stabilise initial phases of training and improve overall performance. \citet{muller2023certifiedtrainingsmallboxes} propose SABR, which computes the IBP bounds over a small adversarially selected region rather than over the full perturbation set. It defines a smaller box of radius $\tau = \lambda_{\tau}\varepsilon$, with $\lambda_{\tau} \in [0,1]$, centred around a projected adversarial point, ensuring the local region $\mathcal{B}_{\tau}(\mathbf{x}_{\tau}) = \{\mathbf{x}' : \|\mathbf{x}'-\mathbf{x}_{\tau}\|_{\infty} \leq \tau\}$ is fully contained within $\mathcal{B}_{\varepsilon}(\mathbf{x})$. Final IBP bounds are then computed over $\mathcal{B}_{\tau}(\mathbf{x}_{\tau})$ and used to calculate the final IBP loss, \ie $\mathcal{L}_{\mathrm{SABR}} =\mathcal{L}(-\underline{\mathbf{z}}^{\Delta_\tau}_{\theta}(\mathbf{x},y), y)$. The intuition is that propagating a smaller adversarially selected region reduces the approximation error induced by IBP bounds, which grows with respect to network depth \cite{mao2024}, thereby providing an unsound but effective approximation of the robust loss over the full perturbation set. \citet{mao2023} introduced TAPS, which splits a network into two sub-parts before training the first part through IBP and the second through AT. The idea is that the PGD under-approximation can compensate for some of the over-approximation error induced by IBP bounds. STAPS \cite{mao2023} carries out an identical procedure, only changing the perturbation region to a smaller adversarially selected one as in SABR. \citet{palma2024} formalise a family of \textit{expressive losses}, which interpolate between empirical lower bounds and certified upper bounds of the worst-case robust loss. MTL-IBP is one such method that achieved relative state-of-the-art performance through leveraging a task balancing coefficient $\alpha \in [0,1]$ for fine-grained control over a linear interpolation of PGD and IBP based training, \ie $\mathcal{L}_{\mathrm{MTL-IBP}} = (1-\alpha)\mathcal{L}_\mathrm{PGD} + \alpha \mathcal{L}_\mathrm{IBP}$. \citet{depalma2026} introduced CC-DIST, which splits the network into a feature extractor and classification head, using an empirically trained teacher for feature-space distillation before learning final classifications through their prior CC-IBP objective \cite{palma2024} on the full model. CC-DIST is closest to our work, since it distils knowledge from an empirically robust teacher. However, CC-DIST is tightly coupled with the notion of expressivity \cite{palma2024}, combining a robust feature-space distillation loss with the expressive CC-IBP objective. Moreover, CC-DIST calculates both adversarial examples and IBP bounds for distillation. In contrast, AD-CERT uses a simple logit-level AD branch as the empirical endpoint, while keeping the certified branch as the baseline IBP loss.
\\ \\
All of the previously mentioned works train standard feed-forward ReLU networks. A related but distinct line of work instead studies robustness by construction. \textsc{Sortnet} \cite{zhang2022boostingcertifiedrobustnesslinfinity} is a notable such model that utilises $\ell_\infty$-distance functions to provide a 1-Lipschitz architecture. While \textsc{Sortnet} performs particularly well on the CIFAR-10 dataset with $\varepsilon=\frac{8}{255}$, it is less competitive with standard certified training methods on ReLU architectures, particularly against smaller $\varepsilon$ values (see §\ref{sub-sec:main-results}). The non-smoothness and gradient sparsity introduced by 1-Lipschitz architectures leave room for exploration in their practical applicability to certified training but remains outside the scope of this work.
\section{Conclusions}
\label{sec:conclusions}
In this paper, we introduced AD-CERT, a novel certified training algorithm that combines the empirical strengths of an adversarially trained teacher with the verifiability of IBP bounds through a simple linear combination of adversarial distillation over the logit-level and an IBP loss objective. AD-CERT achieves overall state-of-the-art certified accuracy, outperforming all prior methods across five standard certified training benchmarks, while remaining competitive in standard accuracy. Our ablations further suggest that adversarial logit-level distillation provides an effective empirical endpoint for certified training, acting analogously to previous hard-label lower bound endpoints while improving certified performance. Nevertheless, a substantial gap remains between empirically robust teachers and certifiably robust students in terms of standard and adversarial accuracy, particularly on harder settings, highlighting the need for future work on stronger mechanisms for transferring empirical robustness into certifiable robustness.
\\ \\
\paragraph{Acknowledgements} This work is partially supported by the NICYBER2025 programme, funded by Innovate UK, and partially by a DfE studentship. The views expressed are those of the authors and do not necessarily represent the project or the funding agency.

\bibliography{refs}

\newpage
\appendix
\section{Interval Bound Propagation (IBP)} \label{app:ibp}
IBP has been at the core of \textit{all} recent state-of-the-art certified training methods \cite{shi2021fastcertifiedrobusttraining,palma2022,muller2023certifiedtrainingsmallboxes, mao2023, palma2024, depalma2026}. Here, we formalise the bound propagation for a single layer, as proposed by \citet{gowal2018}. For an affine layer (\eg convolutional or linear layer) represented by $h_k(z_{k-1}) = Wz_{k-1}+b$, we can define the bound calculation as follows:
\begin{align}
  \mu_{k-1} = \frac{\bar{z}_{k-1} + \underline{z}_{k-1}}{2} &, \quad \mu_k = W\mu_{k-1} + b, \\
  r_{k-1} = \frac{\bar{z}_{k-1} - \underline{z}_{k-1}}{2} &, \quad r_k = |W|r_{k-1}, \\
  \underline{z}_k = \mu_k - r_k &, \quad \bar{z}_k = \mu_k + r_k,  
\end{align}
where $\mu$ denotes the centre of the interval box, $r$ its radius, and
$\underline{z}_k,\bar{z}_k$ denote lower and upper bounds of the layer output, respectively. When $h_k$ is an element-wise increasing activation function, \eg ReLU, sigmoid, or tanh, we instead trivially propagate the bounds as $\underline{z}_k = h_k(\underline{z}_{k-1})$ and $\bar{z}_k = h_k(\bar{z}_{k-1})$. For ReLU networks, an unstable ReLU activation implies the ReLU activation may be either active or inactive given the input interval, \ie $\underline{z}_{k-1}< 0 <\bar{z}_{k-1}$. Hence, incomplete verifiers must over-approximate the possible activation outputs. Figure \ref{fig:ibp_relu_comparison}
compares the loose IBP box relaxation with the tighter convex hull relaxation for an unstable ReLU.
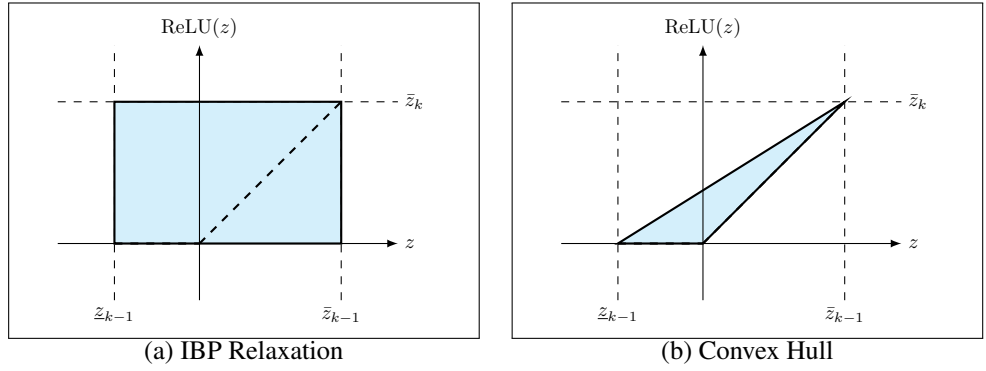
\begin{figure}[h]
\centering
\begin{tabular}{cc}
\bmvaHangBox{%
\fbox{%
\parbox{6.0cm}{%
\centering
\begin{tikzpicture}[>=latex, scale=0.75, transform shape]
    \def\u{2.5}
    \def\l{-1.5}
    \def\h{\u}
    \fill[cyan!15] (\l,0) rectangle (\u,\h);
    \draw[thick] (\l,0) -- (\u,0) -- (\u,\h) -- (\l,\h) -- cycle;
    \draw[->] (-2.5, 0) -- (3.5, 0) node[right] {$z$}; 
    \draw[->] (0, -1) -- (0, 3.5) node[above] {$\mathrm{ReLU}(z)$};
    \draw[dashed] (\l, -1) -- (\l, 3.5) node[at start, below] {$\underline{z}_{k-1}$};
    \draw[dashed] (\u, -1) -- (\u, 3.5) node[at start, below] {$\bar{z}_{k-1}$};
    \draw[dashed] (-2.5, \h) -- (3.5, \h)
        node[right] {$\bar{z}_k $};
    \draw[dashed, thick, black] (0,0) -- (\u,\u);
    \draw[dashed, thick, black] (\l,0) -- (0,0);
\end{tikzpicture}
}}}
&
\bmvaHangBox{%
\fbox{%
\parbox{6.0cm}{%
\centering
\begin{tikzpicture}[>=latex, scale=0.75, transform shape]
    \def\u{2.5}
    \def\l{-1.5}
    \def\h{\u}
    \fill[cyan!15] (\l,0) -- (0,0) -- (\u,\u) -- cycle;
    \draw[thick] (\l,0) -- (0,0) -- (\u,\u) -- cycle;
    \draw[->] (-2.5, 0) -- (3.5, 0) node[right] {$z$}; 
    \draw[->] (0, -1) -- (0, 3.5) node[above] {$\mathrm{ReLU}(z)$};
    \draw[dashed] (\l, -1) -- (\l, 3.5) node[at start, below] {$\underline{z}_{k-1}$};
    \draw[dashed] (\u, -1) -- (\u, 3.5) node[at start, below] {$\bar{z}_{k-1}$};
    \draw[dashed] (-2.5, \h) -- (3.5, \h)
        node[right] {$\bar{z}_k$};
    \draw[dashed, thick, black] (0,0) -- (\u,\u);
    \draw[dashed, thick, black] (\l,0) -- (0,0);
\end{tikzpicture}
}}}
\\
(a) IBP Relaxation & (b) Convex Hull
\end{tabular}
\caption{Comparison of IBP relaxation and the tighter convex hull relaxation. The shaded regions represent convex over-approximations of an unstable ReLU function. The dashed line represents the actual possible activation outputs.}
\label{fig:ibp_relu_comparison}
\end{figure}
\section{Deferred Proofs of Theoretical Results}\label{app:proofs}
In this section, we provide proofs for each of the theoretical results from our methodology section (§\ref{sec:methodology}) of the main text. In each case, we restate the claim for the ease of the reader.
\begin{proposition}[Equivalence of AD and soft-label AT]
Let $\mathbf{x_{adv}}$ be fixed, and let the teacher predictive distribution $P_{\theta_t}(\mathbf{x})$ be fixed. Then the AD objective:
\begin{equation}
    \mathcal{L}_{\mathrm{AD}}
    :=
    \mathrm{KL}\!\left(P_{\theta_t}(\mathbf{x}) \,\|\, P_{\theta_s}(\mathbf{x_{adv}})\right),
\end{equation}
is equal, up to an additive constant independent of the student parameters, to AT with soft teacher labels, namely:
\begin{equation}
    \mathcal{L}_{\mathrm{soft\mbox{-}AT}}
    :=
    \mathcal{L}_\mathrm{CE}\!\left(f_{\theta_s}(\mathbf{x_{adv}}),\, P_{\theta_t}(\mathbf{x})\right).
\end{equation}
Moreover, both objectives induce identical backpropagation updates with respect to the student logits. In particular, for each class index $i$:
\begin{equation}
    \frac{\partial \mathcal{L}_{\mathrm{AD}}}{\partial z_{\theta_s}^i}
    =
    [P_{\theta_s}(\mathbf{x_{adv}})]_i - [P_{\theta_t}(\mathbf{x})]_i
    =
    \frac{\partial \mathcal{L}_{\mathrm{soft\mbox{-}AT}}}{\partial z_{\theta_s}^i}.
\end{equation}
\end{proposition}
\begin{proof}
For brevity, we let $z^i_{\theta_s}:= [f_{\theta_s}(\mathbf{x_{adv}})]_i$ and $P^i_{\theta_s}:= [P_{\theta_s}]_i$ denote the i-th entry of the raw logit outputs and predictive distributions of the student classifier, respectively. First, recall the derivative of the softmax function:
\begin{equation}
    \frac{\partial P^j}{\partial z^i} = 
\begin{cases} 
P^i(1 - P^i), & i=j, \\ 
-P^i P^j, & i \neq j. 
\end{cases} 
\end{equation}
For ease of notation, we can rewrite this in terms of the Kronecker delta function:
\begin{equation}
    \frac{\partial P^j}{\partial z^i} = P^j(\delta_{ij} - P^i), 
\quad \text{where } \delta_{ij} = 
\begin{cases} 
1, & i=j, \\ 
0, & i \neq j .
\end{cases}
\end{equation}
Now, we compare the derivatives of $\mathcal{L}_{\mathrm{AD}}$ and $\mathcal{L}_{\mathrm{soft\mbox{-}AT}}$ with respect to $z^i_{\theta_s}$. Given that $\mathcal{L}_{\mathrm{AD}} = \sum^C_{j=1} P_{\theta_t}^j(\mathbf{x}) \cdot \log \left( \frac{ P_{\theta_t}^j(\mathbf{x})}{ P_{\theta_s}^j(\mathbf{x_{adv}})} \right) = \sum_{j=1}^{C} P_{\theta_t}^j(\mathbf{x}) \cdot \log(P_{\theta_t}^j(\mathbf{x})) - \sum_{j=1}^{C} P_{\theta_t}^j(\mathbf{x}) \cdot \log(P_{\theta_s}^j(\mathbf{x_{adv}}))$, we note that the $\sum_{j=1}^{C} P_{\theta_t}^j(\mathbf{x}) \cdot \log(P_{\theta_t}^j(\mathbf{x}))$ term is independent of student parameters, so:
\begin{align}
     \frac{\partial \mathcal{L}_{\mathrm{AD}}}{\partial z_{\theta_s}^i} &= \frac{\partial}{\partial z_{\theta_s}^i} \left[ - \sum_{j=1}^{C} P_{\theta_t}^j(\mathbf{x}) \log(P_{\theta_s}^j(\mathbf{x_{adv}})) \right],
     \label{eq:partial_ard}\\
    &= - \sum_{j=1}^{C} P_{\theta_t}^j(\mathbf{x})  \frac{1}{P_{\theta_s}^j(\mathbf{x_{adv}})}  \frac{\partial P_{\theta_s}^j(\mathbf{x_{adv}})}{\partial z_s^i}, \\
    &= - \sum_{j=1}^{C} \frac{P_{\theta_t}^j(\mathbf{x})}{P_{\theta_s}^j(\mathbf{x_{adv}})}  P_{\theta_s}^j(\mathbf{x_{adv}})(\delta_{ij} - P_{\theta_s}^i(\mathbf{x_{adv}})), \\
    &= - \sum_{j=1}^{C} P_{\theta_t}^j(\mathbf{x}) (\delta_{ij} - P_{\theta_s}^i(\mathbf{x_{adv}})), \\
    &= - P_{\theta_t}^i(\mathbf{x}) + P_{\theta_s}^i(\mathbf{x_{adv}}) \sum_{j=1}^{C} P_{\theta_t}^j(\mathbf{x}), \\
    &= P_{\theta_s}^i(\mathbf{x_{adv}}) - P_{\theta_t}^i(\mathbf{x}). \label{eq:AD_min}
\end{align}
Finally, recall that $\mathcal{L}_{\mathrm{soft\mbox{-}AT}} = \mathcal{L}_\mathrm{CE}(f_{\theta_s}(\mathbf{x_{adv}}), P_{\theta_t}(\mathbf{x}))$, and,
\begin{equation}
    \mathcal{L}_\mathrm{CE}(f_{\theta_s}(\mathbf{x_{adv}}), P_{\theta_t}(\mathbf{x})) =-\sum^C_{j=1} P_{\theta_t}^j(\mathbf{x}) \log(P_{\theta_s}^j(\mathbf{x_{adv}})).
\end{equation}
Hence, the result falls immediately from Equations \eqref{eq:partial_ard}-\eqref{eq:AD_min}, and by extension, the two objectives must only differ by an additive constant independent of $\theta_s$.
\end{proof}

\begin{proposition}[Analytical form of AD-CERT]
Let $H(P)$ and $H(Q,P)$ denote the entropy and cross entropy respectively for probability distributions $P$ and $Q$. That is:
\begin{equation}
    H(P) =  -\sum_{i=1}^C p_i\log(p_i), \qquad H(Q,P) = -\sum_{i=1}^C q_i\log(p_i).
\end{equation}
Then, $\mathcal{L}_{\mathrm{AD\mbox{-}CERT}}$ from Equation \eqref{eq:ad_cert_loss} can be re-written as:
\begin{equation}
    \mathcal{L}_{\mathrm{AD\mbox{-}CERT}}
    =
    (1-\alpha)H(P_{\theta_t}(\mathbf{x}),P_{\theta_s}(\mathbf{x_{adv}}))
    +
    \alpha \mathcal{L}_{\mathrm{IBP}}
    -
    (1-\alpha)H(P_{\theta_t}(\mathbf{x})).
\end{equation}
Therefore, up to a constant independent of $\theta_s$, AD-CERT is a scalarisation between
(i) a soft-label adversarial lower bound surrogate and
(ii) a certified IBP upper bound.
\end{proposition}
\begin{proof}
    The proof is immediate since we can write:
    \begin{align}
        \mathcal{L}_{\mathrm{KL}}\!\left(P_{\theta_t}(\mathbf{x}) \,\|\, P_{\theta_s}(\mathbf{x_{adv}})\right) &= \sum^C_{i=1} P_{\theta_t}^i(\mathbf{x}) \cdot \log \left( \frac{ P_{\theta_t}^i(\mathbf{x})}{ P_{\theta_s}^i(\mathbf{x_{adv}})} \right), \\
        &= \sum^C_{i=1} P_{\theta_t}^i(\mathbf{x}) \cdot \big ( \log(P_{\theta_t}^i(\mathbf{x})) - \log(P_{\theta_s}^i(\mathbf{x_{adv}})) \big ), \\
        &= - \sum^C_{i=1} P_{\theta_t}^i(\mathbf{x}) \log(P_{\theta_s}^i(\mathbf{x_{adv}})) + \sum^C_{i=1} P_{\theta_t}^i(\mathbf{x})\log(P_{\theta_t}^i(\mathbf{x})),\\
        &=  H(P_{\theta_t}(\mathbf{x}),P_{\theta_s}(\mathbf{x_{adv}})) - H(P_{\theta_t}(\mathbf{x})).
    \end{align}
    Hence, substituting this identity into Equation \eqref{eq:ad_cert_loss} gives: 
    \begin{equation}
        \mathcal{L}_{\mathrm{AD\mbox{-}CERT}}
        =
        (1-\alpha)H(P_{\theta_t}(\mathbf{x}),P_{\theta_s}(\mathbf{x_{adv}})) +\alpha \mathcal{L}_{\mathrm{IBP}}
        -(1-\alpha)H(P_{\theta_t}(\mathbf{x})),
    \end{equation}
    as required.
\end{proof}
\begin{lemma}
For AD, the optimal logit margin, $\Delta z_{ij}^*$, between any two classes $i,j$ is:
\begin{equation}
\Delta z_{ij}^* = \log\left(\frac{[P_{\theta_t}(\mathbf{x})]_i}{[P_{\theta_t}(\mathbf{x})]_j}\right).
\end{equation}
Moreover, standard AT has no finite minimiser, and its infimum $0$ is approached only as $[f_{\theta_s}(\mathbf{x_{adv}})]_y-[f_{\theta_s}(\mathbf{x_{adv}})]_j\to+\infty, \quad \forall j\neq y$.
\end{lemma}
\begin{proof}
First note that, from Equation \eqref{eq:AD_min}, $\mathcal{L}_{\mathrm{AD}}$ is minimised iff $P_{\theta_s}=P_{\theta_t}$,
or equivalently, $P_{\theta_s}^i(\mathbf{x_{adv}})=P_{\theta_t}^i(\mathbf{x})$ $\forall$ $i\in\{1,\dots,C\}.$
Rewriting this using the definition of softmax, we have:
\begin{equation}
     \frac{e^{z_s^i}}{\sum_{j=1}^C e^{z_s^j}} = P_{\theta_t}^i(\mathbf{x}) \quad \forall \ i\in\{1,\dots,C\},
\end{equation}
where $z^i_{\theta_s} = [f_{\theta_s}(\mathbf{x_{adv}})]_i$, which implies:
\begin{equation}
    z_s^i = \log \left(P_{\theta_t}^i(\mathbf{x}) \cdot \sum_{j=1}^C e^{z_s^j} \right) = \log \left(P_{\theta_t}^i(\mathbf{x}) \right) + \log \left(\sum_{j=1}^C e^{z_s^j} \right)
\end{equation}
Then, we can write the optimal logit for a given class $i$ as $z_s^{i*}:=\log(P_{\theta_t}^i(\mathbf{x}))+c$ for some shared constant, $c\in\mathbb{R}$.
Since softmax is invariant to adding a constant to each logit, the optimal logit margin between any two classes $i$ and $j$ is:
\begin{equation}
   \Delta z_{ij}^* := z_s^{i*}-z_s^{j*} = \log(P_{\theta_t}^i(\mathbf{x}))-\log(P_{\theta_t}^j(\mathbf{x})) = \log\left(\frac{P_{\theta_t}^i(\mathbf{x})}{P_{\theta_t}^j(\mathbf{x})}\right),
\end{equation}
which is finite since $P_{\theta_t}^i(\mathbf{x})>0$ $\forall$ $i$. \\
In contrast, standard AT loss is defined by $\mathcal{L}_{AT} = -\log(P_{\theta_s}^y(\mathbf{x_{adv}}))$, with $y$ denoting the index of the true class label here. Hence, we have:
\begin{equation}
    \mathcal{L}_{AT} = -\log\left(\frac{e^{z_s^y}}{\sum_{j=1}^C e^{z_s^j}}\right) = \log\left(\sum_{j=1}^C e^{z_s^j-z_s^y}\right) = \log\left(1+\sum_{j\neq y} e^{z_s^j-z_s^y}\right).
\end{equation}
For any finite logits, we have $e^{z_s^j-z_s^y}>0$, and hence
\begin{equation}
    1+\sum_{j\neq y} e^{z_s^j-z_s^y} > 1 \implies \mathcal{L}_{AT} > 0.
\end{equation}
Therefore, $\mathcal{L}_{AT}$ has no finite minimiser. Since we have established that $\mathcal{L}_{AT} > 0$ for all finite logits, its global minimum is never attained. However, we can show that its infimum is exactly 0. Let $m = \displaystyle \min_{j \neq y} (z_s^y - z_s^j)$ denote the minimum logit margin of the correct class over all incorrect classes. We can bound the loss as follows:
\begin{equation}
    0 < \mathcal{L}_{AT} = \log\left(1+\sum_{j\neq y} e^{-(z_s^y-z_s^j)}\right) \leq \log\left(1 + (C-1)e^{-m}\right),
\end{equation}
where $C$ is the total number of classes. Taking the limit as the minimum margin grows to infinity ($m \to +\infty$), we have:
\begin{equation}
    \lim_{m \to +\infty} \log\left(1 + (C-1)e^{-m}\right) = \log(1 + 0) = 0.
\end{equation}
Thus, $\inf \mathcal{L}_{AT} = 0$. This infimum is not reachable for any finite network weights, but is approached asymptotically as the correct-class logit margins tend to positive infinity.
\end{proof}

\section{Additional Results}\label{app:additional_results}
In this section, we present additional theoretical results complementary to those in the main text.
\begin{proposition}[Teacher-margin transfer in logit space]
Assume the teacher $f_{\theta_t}$ predicts the true class $y$ on the natural input $\mathbf{x}$. Then, if:
\begin{equation}
   \left\|f_{\theta_s}(\mathbf{x_{adv}}) - f_{\theta_t}(\mathbf{x})\right\|_\infty < \frac{1}{2} \min_{j\neq y} [\mathbf{z}_{\theta_t}(\mathbf{x},y)]_j, 
\end{equation}
it follows that $\displaystyle \argmax_i [f_{\theta_s}(\mathbf{x_{adv}})]_i = y$.
\end{proposition}
\begin{proof}
    First, define $\delta :=\left\|f_{\theta_s}(\mathbf{x}_{\mathrm{adv}}) - f_{\theta_t}(\mathbf{x})\right\|_\infty$. Then, $-\delta \le [f_{\theta_s}(\mathbf{x_{adv}})]_i - [f_{\theta_t}(\mathbf{x})]_i \le \delta,$ for any class index $i$. Which means, $[f_{\theta_s}(\mathbf{x_{adv}})]_y  \ge [f_{\theta_t}(\mathbf{x})]_y - \delta$ and $
        [f_{\theta_s}(\mathbf{x_{adv}})]_i  \le [f_{\theta_t}(\mathbf{x})]_i + \delta$ for any $i \ne y$. Moreover:
    \begin{align}
        \implies [f_{\theta_s}(\mathbf{x_{adv}})]_y - [f_{\theta_s}(\mathbf{x_{adv}})]_i &\ge [f_{\theta_t}(\mathbf{x})]_y - [f_{\theta_t}(\mathbf{x})]_i - 2\delta, \\
        \implies [\mathbf{z}_{\theta_s}(\mathbf{x_{adv}},y)]_i &\ge [\mathbf{z}_{\theta_t}(\mathbf{x},y)]_i - 2\delta.
    \end{align}
    Hence, if $\delta < \frac{1}{2} \min_{i\neq y} [\mathbf{z}_{\theta_t}(\mathbf{x},y)]_i$, then clearly $[\mathbf{z}_{\theta_t}(\mathbf{x},y)]_i - 2\delta > 0$, which implies \\$[\mathbf{z}_{\theta_s}(\mathbf{x_{adv}},y)]_i > 0$ and $\displaystyle \argmax_i [f_{\theta_s}(\mathbf{x}_{\mathrm{adv}})]_i = y$.
\end{proof}

\begin{lemma}
Let $\mathcal{L}_{AT}$ be bounded by some $U \in \mathbb{R}$ and $U > 0$, \ie
\begin{equation}
    \mathcal{L}_{AT} = \mathcal{L}_\mathrm{CE}(f_{\theta_s}(\mathbf{x_{adv}}), y) = -\log(P_{\theta_s}^y) \le U,
\end{equation}
where $y \in \{1, \dots, C\}$ is the true label index and $P_{\theta_s}^y = \left[\mathrm{softmax}(f_{\theta_s}(\mathbf{x_{adv}}))\right]_y$. Then, for every $j \neq y$, the corresponding true-vs-false logit margin satisfies
\begin{equation}
   z_s^y - z_s^j \ge -\log(e^U - 1). 
\end{equation}
\end{lemma}
\begin{proof}
Notice that AT loss can be written as:
\begin{equation}
    \mathcal{L}_{AT}
    =
    -\log\left(\frac{e^{z_s^y}}{\sum_{k=1}^C e^{z_s^k}}\right)
    =
    \log\left(1+\sum_{j\neq y} e^{z_s^j-z_s^y}\right).
\end{equation}
Now assume that $\mathcal{L}_{AT} \le U$. Then,
\begin{align}
    \log\left(1+\sum_{j\neq y} e^{z_s^j-z_s^y}\right) &\le U, \\
    \implies 1+\sum_{j\neq y} e^{z_s^j-z_s^y} &\le e^U, \\
    \implies \sum_{j\neq y} e^{z_s^j-z_s^y} &\le e^U-1.
\end{align}
Since each term in the sum is non-negative, it follows that for every $j\neq y$, we have:
\begin{align}
    e^{z_s^j-z_s^y} &\le e^U-1, \\
    \implies z_s^j-z_s^y &\le \log(e^U-1), \\
    \implies z_s^y-z_s^j &\ge -\log(e^U-1),
    \qquad \forall j\neq y,
\end{align}
as required. 
\end{proof}

\section{Additional Experiments}
\label{app:additional-experiments}
In the following section we include some complementary experiments to those found in the main text.
\\ \\
Here we evaluate the sensitivity of AD-CERT and CC-DIST \cite{depalma2026} to teachers trained under different adversarial training methods and schedules. The sensitivity results from Table \ref{tab:adcert_ccdist_teacher_sensitivity} suggest that the choice of teacher can affect the standard and adversarial accuracy of the resulting students, but the certified accuracy is comparatively stable across teacher variants, particularly for AD-CERT. Here, short-cycle refers to the training procedure described in \citet{depalma2026}, where we use the SGD optimiser with momentum set to 0.9, a 30-epoch training cycle with a cyclic learning rate, which linearly increases from 0 to 0.2 during the first half of the training, and then decreases back to 0. Long-cycle corresponds to the standard AT settings used in the CTBench library \cite{mao2025}, where PGD \cite{madry2018deeplearningmodelsresistant} and EDAC \cite{zhang2025} models are trained for 240 epochs using the Adam optimiser. Short-cycle teachers tend to slightly improve the certified results of AD-CERT and CC-DIST, which coincides with the observations of \citet{depalma2026} in their work. In general, CC-DIST appears to benefit more noticeably from short-cycle teachers, with AD-CERT being slightly less sensitive. Nevertheless, AD-CERT remains consistently stronger in certified and AA across all teacher protocols, while CC-DIST tends to achieve slightly higher standard accuracy in most settings, consistent with previous observations in §\ref{sub-sec:main-results}.
\begin{table}[h]
\centering
\caption{Sensitivity analysis of AD-CERT and CC-DIST on CIFAR-10 at $\varepsilon=\frac{8}{255}$ with different teacher training protocols.}
\label{tab:adcert_ccdist_teacher_sensitivity}
\footnotesize
\begin{tabular}{l c c l c c c}
\toprule
\multicolumn{1}{c}{Teacher} 
& Teach. Std. [\%] 
& Teach. AA [\%] 
& \multicolumn{1}{c}{Method} 
& Std. [\%] 
& AA [\%] 
& Cert. [\%] \\ 
\cmidrule{1-7}
\multirow{2}{*}{PGD-Long}   
& \multirow{2}{*}{78.71} 
& \multirow{2}{*}{35.93} 
& AD-CERT & \textbf{53.17} & \textbf{35.95} & \textbf{35.21} \\
& & & CC-DIST & 52.35 & 34.80 & 33.75 \\
\cmidrule{1-7}
\multirow{2}{*}{PGD-Short}  
& \multirow{2}{*}{76.14} 
& \multirow{2}{*}{42.55} 
& AD-CERT & 53.90 & \textbf{36.52} & \textbf{35.61} \\
& & & CC-DIST & \textbf{54.44} & 36.05 & 34.81 \\
\cmidrule{1-7}
\multirow{2}{*}{EDAC-Long}  
& \multirow{2}{*}{78.95} 
& \multirow{2}{*}{42.48} 
& AD-CERT & 53.49 & \textbf{36.46} & \textbf{35.60} \\
& & & CC-DIST & \textbf{54.32} & 35.69 & 34.55 \\
\cmidrule{1-7}
\multirow{2}{*}{EDAC-Short} 
& \multirow{2}{*}{73.55} 
& \multirow{2}{*}{41.99} 
& AD-CERT & 53.69 & \textbf{36.69} & \textbf{35.84} \\
& & & CC-DIST & \textbf{54.33} & 35.76 & 34.68 \\
\bottomrule
\end{tabular}
\end{table}
\\
Table \ref{tab:ad-cert-seed-results} reports the 3-seed stability of AD-CERT compared to previous SOTA certified training methods, CC-DIST \cite{depalma2026} and MTL-IBP \cite{palma2024}, on CIFAR-10 at $\varepsilon=\frac{8}{255}$. AD-CERT obtains the highest mean accuracies among the compared methods while also reducing seed variance. These results are also favourable and consistent when compared to the CTBench 3-seed baselines \cite{mao2025} (Table 16), where all comparable methods exhibit larger variance and lower mean accuracies.
\begin{table}[h]
\centering
\caption{3-seed stability on CIFAR-10 at $\varepsilon=8/255$.}
\label{tab:ad-cert-seed-results}
\footnotesize
\begin{tabular}{lccc}
\toprule
Method & Std. [\%] & AA [\%] & Cert. [\%] \\
\midrule
AD-CERT & $\mathbf{53.76 \pm 0.09}$ & $\mathbf{36.75 \pm 0.07}$ & $\mathbf{35.78 \pm 0.15}$ \\
CC-DIST & $53.73 \pm 0.54$ & $35.67 \pm 0.28$ & $34.68 \pm 0.32$ \\
MTL-IBP & $53.54 \pm 0.68$ & $35.42 \pm 0.30$ & $35.05 \pm 0.90$ \\
\bottomrule
\end{tabular}
\end{table}
\\
We provide an additional loss ablation on CIFAR-10 at $\varepsilon=\frac{8}{255}$ in Table \ref{tab:loss-abs} which compares our proposed approach outlined in our methodology (§\ref{sub-sec:ad-cert}) to two alternative formulations of AD-CERT, namely, using a clean teacher in place of our adversarially trained teacher, and passing $\mathbf{x_{adv}}$ into the teacher. Moreover, we include a comparison to CC-IBP \cite{palma2024}, another previous SOTA certified training method which corresponds to the non-distillation component of \citet{depalma2026}'s CC-DIST (see Equation \eqref{eq:cc-dist}). Importantly, our proposed approach outperforms both alternative formulations of AD-CERT across standard, AA and certified accuracy, further justifying our methodology.
\begin{table}[h]
\centering
\caption{Additional loss ablations on CIFAR-10 at $\varepsilon=8/255$.}
\label{tab:loss-abs}
\footnotesize
\begin{tabular}{lccc}
\toprule
Method & Std. [\%] & AA [\%] & Cert. [\%] \\
\midrule
AD-CERT & 53.69 & \textbf{36.69} & \textbf{35.85} \\
Clean teacher & 53.31 & 36.34 & 35.37 \\
Teacher on $\mathbf{x_{adv}}$ & 51.79 & 35.84 & 35.03 \\
\midrule
CC-IBP & \textbf{53.80} & 35.57 & 34.80 \\
\bottomrule
\end{tabular}
\end{table}
\\
Figure \ref{fig:tradeoff_curve} provides a complementary trade-off analysis to Figure \ref{fig:alpha_sensitivity_tinyimagenet} from the main paper. We include an adaptation of a Pareto front style plot in our fixed perturbation radius setting, where we vary the IBP coefficient of several previous SOTA certified training methods in the range $[0, 1]$. The findings are similar to those discussed in the main text. AD-CERT outperforms all previous ReLU-based training methods in certified accuracy while remaining competitive in standard accuracy, particularly as $\alpha$ is increased.
\begin{figure}[h]
    \centering
    \includegraphics[width=0.85\linewidth]{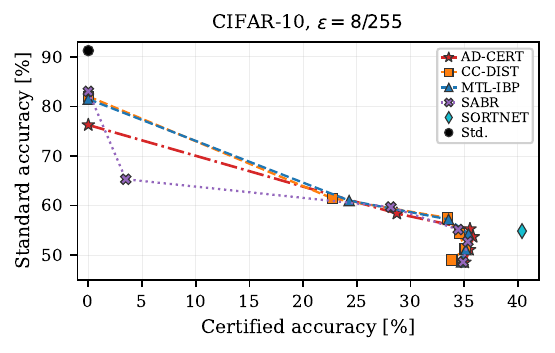}
    \vspace{-0.2cm}
    \caption{
    Trade-off analysis of standard-certified accuracy for several SOTA certified training methods with varying IBP coefficients in the range $[0, 1]$.
    }
    \vspace{-0.5cm}
    \label{fig:tradeoff_curve}
\end{figure}
\section{\scshape Pseudo-code}\label{app:pseudo}
In this section, we outline the pseudocode for the training procedure of AD-CERT (§\ref{sub-sec:ad-cert}) in Algorithm \ref{alg:ad_cert_train}.

\begin{algorithm}[h!]
\caption{AD-CERT Training Procedure}\label{alg:ad_cert_train}
\begin{algorithmic}[1]

\Require PGD teacher $f_{\theta_t}$, data $\mathcal{D}$, perturbation radius ${\varepsilon}_{train}$, epochs $N_{warm}, N$, IBP coefficient $\alpha$
\Statex
\State \textbf{Initialisation}
\State Initialise student $f_{\theta_s}$ with parameters ${\theta_s}$ \Comment{\citet{shi2021fastcertifiedrobusttraining} initialisation}
\State Detach teacher $f_{\theta_t}$ gradients from training
\Statex
\State \textbf{Training}
\For{$epoch = 1$ \textbf{to} $N$}
    \If{$epoch \le N_{warm}$}
        \State Update $\varepsilon_{curr}$ and $\alpha_{curr}$ \Comment{ramp-up to $\varepsilon_{train}, \alpha$} 
    \EndIf
    \For{$(\mathbf{x}, y)$ in $\mathcal{D}$}
        \State Compute $\mathbf{x_{adv}}$ via PGD attack
        \State Compute IBP bounds $\underline{\mathbf{z}}^\Delta_{{\theta_s}}(\mathbf{x}, y)$ using $\varepsilon_{curr}$
        
        \State $\mathcal{L}_{dist} \gets \mathrm{KL}\left( \mathrm{softmax}(f_{\theta_t}(\mathbf{x})) \,\|\, \mathrm{softmax}(f_{\theta_s}(\mathbf{x_{adv}})) \right)$
        \State $\mathcal{L}_{robust} \gets \mathcal{L}_\mathrm{CE}(-\underline{\mathbf{z}}^\Delta_{{\theta_s}}(\mathbf{x}, y), y)$
        
        \State $\mathcal{L}_{total} \gets (1-\alpha_{curr}) \mathcal{L}_{dist} + \alpha_{curr} \mathcal{L}_{robust}$
        
        \State Update ${\theta_s} \gets {\theta_s} - \eta \nabla_{\theta_s} \mathcal{L}_{total}$
    \EndFor
\EndFor
\State \Return $f_{\theta_s}$

\end{algorithmic}
\end{algorithm}

\section{\scshape Experimental Setup}\label{app:exp-details}
\subsection{Dataset}
We conduct experiments on MNIST \cite{lecun2010mnist}, CIFAR-10 \cite{krizhevsky2009}, and TinyImageNet \cite{le2015tiny}. These datasets are open-source and freely available, with unspecified licenses. We follow the data preprocessing directly available in the CTBench library \cite{mao2025}. No preprocessing is applied to MNIST. CIFAR-10 and TinyImageNet are normalised using the dataset mean and standard deviation, and random horizontal flips are applied. For CIFAR-10, additional random cropping to $32 \times 32$ is applied after zero-padding each image by 2 pixels on all sides. For TinyImageNet, random cropping to $64 \times 64$ is applied after zero-padding each image by 4 pixels on all sides. We train on the corresponding training sets and certify on the validation sets, following common practice in the certified training literature \cite{shi2021fastcertifiedrobusttraining, muller2023certifiedtrainingsmallboxes, mao2023, palma2024, mao2025, depalma2026}.
\subsection{Model Architectures}
We use the standard CNN7 architecture, a convolutional network consisting of seven convolutional and linear layers. Each layer, except for the final linear layer, is followed by Batch Normalisation and a ReLU activation. This architecture has been shown to perform consistently well across settings \cite{shi2021fastcertifiedrobusttraining, mao2025}, and is therefore widely adopted in the certified training literature \cite{shi2021fastcertifiedrobusttraining, muller2023certifiedtrainingsmallboxes, mao2023, palma2024, mao2025, depalma2026}. For TinyImageNet, we double the stride of the final convolutional layer to reduce computational cost.
\subsection{Training Details}
\label{sub-sec:train_det}
\paragraph{Implementation}
We implement AD-CERT in CTBench using PyTorch. The training loss follows the CTBench decomposition into a natural loss, a robust loss, and regularisation terms. For AD-CERT, the robust loss is given by Equation \ref{eq:ad_cert_loss}, where the AD term uses a fixed adversarially trained teacher, and the certified term is computed with IBP bounds. Unless stated otherwise, the teacher and student use the same CNN7 architecture.
\paragraph{Initialisation}
Adversarial teacher models are initialised by Kaiming uniform \cite{he2015}, while certified student models are initialised using the IBP initialisation from \citet{shi2021fastcertifiedrobusttraining}.
\paragraph{Training Schedule}
We follow the standard CTBench training schedule. Certified models are first trained for one epoch with $\varepsilon = 0$, and then use a warm-up phase where $\varepsilon$ is smoothly increased from $0$ to the target value. The warmup phase is 20 epochs for MNIST with $\varepsilon = 0.1$ and $\varepsilon = 0.3$, 80 epochs for CIFAR-10 with $\varepsilon = \frac{2}{255}$, 120 epochs for CIFAR-10 with $\varepsilon = \frac{8}{255}$, and 80 epochs for TinyImageNet with $\varepsilon = \frac{1}{255}$. We use the IBP regularisation proposed by \citet{shi2021fastcertifiedrobusttraining}, with weight $0.5$ on MNIST and CIFAR-10, and $0.2$ on TinyImageNet, during warmup. In total, we train for 70 epochs on MNIST, 160 epochs on CIFAR-10 with $\varepsilon = \frac{2}{255}$, 240 epochs on CIFAR-10 with $\varepsilon = \frac{8}{255}$, and 160 epochs on TinyImageNet.
\paragraph{Optimisation}
In general, we follow the exact optimisation setup from CTBench. We use Adam \cite{kingma2015adam} with a learning rate of $5 \times 10^{-4}$. The learning rate is decayed by a factor of $0.2$ at epochs 50 and 60 for MNIST, epochs 120 and 140 for CIFAR-10 with $\varepsilon = \frac{2}{255}$, epochs 200 and 220 for CIFAR-10 with $\varepsilon = \frac{8}{255}$, and epochs 120 and 140 for TinyImageNet. We use a batch size of 256 for MNIST and 128 for CIFAR-10 and TinyImageNet. Gradients are clipped to $10$ in $\ell_2$ norm before every optimiser step. No weight decay is applied, and $L_1$ regularisation is applied only to the weights of linear and convolutional layers. Batch normalisation follows the CTBench population-statistics setting. We apply Stochastic Weight Averaging (SWA) using the settings in accordance with MTL-IBP's training setup in CTBench since both methods use an unsound robust objective, making direct model selection on the robust loss impractical.
\paragraph{Teacher Models}
We use the pre-trained PGD checkpoints provided by CTBench as teacher models, except for CIFAR-10 with $\varepsilon = \frac{8}{255}$, where our ablation teacher sensitivity analysis from Table \ref{tab:adcert_ccdist_teacher_sensitivity} showed that a 30-epoch short-cycle PGD model trained with a cyclic learning-rate schedule exactly as described by \citet{depalma2026} alongside an EDAC \cite{zhang2025} step size of $0.3$, gives marginally better performance. All teachers are kept fixed during AD-CERT training.
\paragraph{Tuning of Hyperparameters}
In general, we perform minimal hyperparameter tuning due to the computational overhead of running complete verification. We ran initial experiments following the exact hyperparameters and settings used in CTBench \cite{mao2025} for the MTL-IBP \cite{palma2024} method. For the tuning of $\alpha$, we use a manually selected search range guided by some reported quantities in the CTBench library, namely, unstable ReLU ratio, IBP certification rate, and standard, adversarial, and certified accuracy. These metrics help indicate whether the objective places too much or too little weight on the IBP branch, allowing us to adjust $\alpha$ accordingly. We find that slightly reducing $\alpha$ is beneficial on MNIST at $\varepsilon=0.1$ and on TinyImageNet at $\varepsilon=\frac{1}{255}$. For TinyImageNet, we also observe improvements from increasing $w_{\mathrm{rob}}$ and enlarging the PGD attack region used in the empirical branch. The final hyperparameters used for AD-CERT across all settings are reported in Table \ref{tab:adcert_hyperparameters}.
\begin{table}[h]
\centering
\caption{Best hyperparameters for AD-CERT across dataset settings.}
\label{tab:adcert_hyperparameters}
\footnotesize
\begin{tabular}{l c c c c c}
\toprule
& \multicolumn{2}{c}{MNIST} 
& \multicolumn{2}{c}{CIFAR-10} 
& TinyImageNet \\
\cmidrule(lr){2-3}
\cmidrule(lr){4-5}
\cmidrule(lr){6-6}
Hyperparameter 
& $0.1$ 
& $0.3$ 
& $\frac{2}{255}$ 
& $\frac{8}{255}$ 
& $\frac{1}{255}$ \\
\midrule
$L_1$ regularisation       & $1\times10^{-6}$ & $1\times10^{-6}$ & $3\times10^{-6}$ & 0.0 & $5\times10^{-5}$ \\
$w_{\mathrm{rob}}$         & 0.7 & 1.0 & 1.0 & 1.0 & 0.9 \\
IBP coefficient ($\alpha$) & $7.5\times10^{-3}$ & 0.5 & 0.01 & 0.5 & $5\times10^{-3}$ \\
Train $\varepsilon$        & 0.2 & 0.3 & $\frac{1}{255}$ & $\frac{8}{255}$ & $\frac{1}{255}$ \\
PGD steps                  & 10 & 1 & 8 & 1 & 1 \\
Attack range scale         & 1.0 & 1.0 & 2.0 & 1.0 & 2.0 \\
\bottomrule
\end{tabular}
\end{table}
\subsection{Certification Details}
\label{sub-sec:cert-details}
Certification is performed within CTBench \cite{mao2025} using its wrapper for the $\alpha,\beta$-CROWN verification library. For each example, we first check whether the model predicts the correct clean label. Incorrectly classified examples are marked as uncertified. For correctly classified examples, we run \textsc{AutoAttack} \cite{croce2020} and if an adversarial counterexample is not found, we attempt certification using a sequence of progressively stronger verifiers. Namely, IBP \cite{gowal2018}, CROWN \cite{zhang2018}, $\alpha$-CROWN \cite{xu2021fast}, and finally, the complete $\alpha,\beta$-CROWN branch-and-bound verifier \cite{wang2021beta}. Each stage is only invoked when the previous stage fails to certify the example. We use a timeout of 400 seconds for the $\alpha,\beta$-CROWN branch-and-bound stage and an overall timeout of 1000 seconds per example.

\subsection{Computational Setup \& Cost}
\label{app:complex}
Table \ref{tab:adcert_runtime} reports the training and certification runtime of AD-CERT under the training and certification setups described in Sections \ref{sub-sec:train_det} and \ref{sub-sec:cert-details}, respectively. All experiments were run on an internal cluster with access to NVIDIA V100, A100, A100 MIG slices, and H100 GPUs. Training jobs were run on either 2g.20GB or 3g.40GB A100 MIG slices, using 8 CPU cores. For certification, we used MIG slices in most cases, but occasionally used full A100 or H100 GPUs for harder settings, such as CIFAR-10 with $\varepsilon = \frac{2}{255}$ and TinyImageNet. For clarity, Table \ref{tab:adcert_runtime} reports runtimes for experiments run on a 3g.40GB A100 MIG slice. Training experiments used 64GB of RAM, while certification experiments used 128GB. The training times range from $\sim$ 1 hour on the simplest MNIST setting, to $\sim$ 1 day on TinyImageNet. Certification times range from $\sim$ 12 minutes to over 3 days. \\ \\
Additionally, we report the training cost complexity associated with an AD-CERT student, assuming a pre-trained teacher, compared to previous adversarial and certified training methods \cite{madry2018deeplearningmodelsresistant, zhang2025, gowal2018, muller2023certifiedtrainingsmallboxes, mao2023, palma2024, depalma2026} in Table \ref{tab:training_costs}. On a high level, AD-CERT corresponds to performing IBP and PGD, with the small added cost of a single forward pass of a teacher model that is detached from training. Note that Table \ref{tab:training_costs} is a direct extrapolation of the table provided in the original CTBench paper \cite{mao2025} with added entries for AD-CERT and CC-DIST.
\begin{table}[h]
\centering
\caption{Training and certification time for AD-CERT on different datasets and $\varepsilon$.}
\label{tab:adcert_runtime}
\footnotesize
\begin{tabular}{l c c c}
\toprule
Dataset & $\varepsilon$ & Train Time (seconds) & Certification Time (seconds) \\
\midrule
\multirow{2}{*}{MNIST}
& $0.1$ & $1.20 \times 10^{4}$ &  $7.67 \times 10^2$ \\
& $0.3$ & $4.08 \times 10^{3}$ & $5.32\times 10^4$ \\
\cmidrule{2-4}
\multirow{2}{*}{CIFAR-10}
& $\frac{2}{255}$ & $2.66 \times 10^{4}$ & $2.61 \times 10^{5} $ \\
& $\frac{8}{255}$ & $1.59 \times 10^{4}$ & $3.58 \times 10^{4}$ \\
\cmidrule{2-4}
TinyImageNet & $\frac{1}{255}$ & $8.97 \times 10^{4}$ & $1.54 \times 10^5$ \\
\bottomrule
\end{tabular}
\end{table}
\begin{table}[h]
\centering
\caption{Detailed breakdown of training costs for each certified training method.}
\footnotesize
\label{tab:training_costs}
\begin{tabular}{l l l}
\toprule
Method & Training cost per batch & Details \\
\midrule
Standard & $T$ & Forward + backward \\
PGD / EDAC & $(M+1)T$ & $M$ attack steps + standard loss computation \\
IBP & $2T$ & Lower and upper bounds propagation \\
SABR & $(M+2)T$ & IBP + PGD \\
MTL-IBP & $(M+2)T$ & IBP + PGD \\
AD-CERT$^\dagger$ & $(M+2)T + F$ & IBP + PGD + detached teacher forward \\
CC-DIST$^\dagger$ & $(M+2)T + F_h$ & CC-IBP + PGD + detached teacher forward on feature split \\
TAPS & $2t + K(M+1)(T-t)$ & IBP for first split and PGD for second split for each class \\
STAPS & $2t + K(M+1)(T-t) + (M+1)T$ & TAPS + PGD \\
\midrule
\multicolumn{3}{c}{Legend} \\
\midrule
$T$ & \multicolumn{2}{l}{Time cost for standard training, including forward + backward pass} \\
$F$ & \multicolumn{2}{l}{Time cost for a forward pass through the teacher network} \\
$F_h$ & \multicolumn{2}{l}{Time cost for a forward pass through the teacher feature extractor network} \\
$M$ & \multicolumn{2}{l}{Number of adversarial attack steps, including repeats} \\
$K$ & \multicolumn{2}{l}{Number of classes} \\
$t$ & \multicolumn{2}{l}{Time cost for standard training in the first network split in TAPS} \\
\bottomrule
\multicolumn{3}{l}{\footnotesize $^\dagger$ Added entries from original table in CTBench \cite{mao2025}.}
\end{tabular}
\end{table}

\section{CC-IBP Distillation (CC-DIST)}\label{app:cc-dist}
\citet{depalma2026} formulate the CC-DIST loss function as a linear combination of an \emph{expressive} robust distillation loss in the feature space, $\mathcal{R}_{f_\theta}(\alpha; \mathbf{x}, y)$, and their prior expressive loss formulation, CC-IBP \cite{palma2024}. Let $f_\theta: \mathbb{R}^{d} \mapsto{\mathbb{R}^C}$ be a classification model as expressed earlier in §\ref{sec:background}. Then, we can rewrite $f_\theta$ as a composition of a feature extractor $h_\theta: \mathbb{R}^{d} \mapsto{\mathbb{R}^k}$ and a classification head $g_\theta: \mathbb{R}^{k} \mapsto{\mathbb{R}^C}$, yielding $f_\theta = g_\theta \circ h_\theta$ (note that we abuse notation of $\theta$ for brevity). Then, the expressive robust distillation loss over the feature space is defined by:
\begin{equation}
    \mathcal{R}_{f_\theta}(\alpha; \mathbf{x}, y) := \sum^k_{i=1} \max \left\{\left([\overline{h}_\theta(\alpha; \mathbf{x})]_i  - [h_{\theta_t}(\mathbf{x})]_i\right)^2, \left([\underline{h}_\theta(\alpha; \mathbf{x})]_i  - [h_{\theta_t}(\mathbf{x})]_i\right)^2\right\},
    \label{eq:rob-dist-loss-ccdist}
\end{equation}
where $\overline{h}_\theta(\alpha; \mathbf{x}) := (1-\alpha)h_\theta(\mathbf{x_{adv}}) + \alpha \overline{h}_\theta(\mathbf{x})$ and  $\underline{h}_\theta(\alpha; \mathbf{x}) := (1-\alpha)h_\theta(\mathbf{x_{adv}}) + \alpha \underline{h}_\theta(\mathbf{x})$, with $\overline{h}_\theta(\mathbf{x})$ and $\underline{h}_\theta(\mathbf{x})$ denoting the IBP upper and lower bounds of $h_\theta(\mathbf{x})$, respectively. Note that $h_{\theta_t}$ denotes the feature extractor of a pre-trained PGD teacher model.
\\ \\
Now, let $\beta$ be the distillation coefficient, determining the relative weight of $\mathcal{R}_{f_\theta}(\alpha; \mathbf{x}, y)$ from Equation \eqref{eq:rob-dist-loss-ccdist}. Omitting any regularisation, the training loss for CC-DIST takes the following form:
\begin{equation}
    \mathcal{L}^{\text{CC-DIST}}_{f_{\theta}}(\alpha, \beta; \mathbf{x}, y) := \ \mathcal{L}^{\text{CC-IBP}}_{f_{\theta}}(\alpha; \mathbf{x}, y) + \beta \ \mathcal{R}_{f_{\theta}}(\alpha; \mathbf{x}, y),
    \label{eq:cc-dist}
\end{equation}
where $\mathcal{L}^{\text{CC-IBP}}_{f_{\theta}}(\alpha; \mathbf{x}, y)$ is the CC-IBP loss from \citet{palma2024}. Denoting $k$ as the dimensionality of the feature space, $\beta = 5/k$ was used for the majority of experiments.
\end{document}